\documentclass{IEEEtran}
\usepackage{cite}
\usepackage{textcomp}
\usepackage{amsmath,amssymb,array,arydshln,color, bm}
\usepackage{amsthm}
\usepackage{cuted}
\usepackage{soul}
\usepackage{flushend}
\usepackage{mathtools}
\usepackage{subcaption}
\usepackage{subfiles}
\usepackage{algpseudocode}
\usepackage{algorithm}
\usepackage{url}
\usepackage{hyperref}
\hypersetup{colorlinks, linkcolor=blue}
\usepackage{comment}
\usepackage{color}
\usepackage{multirow}
\usepackage{etoolbox}

\newtheorem{proposition}{Proposition}
\newtheorem{theorem}{Theorem}

\usepackage[colorinlistoftodos, textwidth=2cm, shadow]{todonotes}

\newcommand{\norm}[1]{\left\lVert#1\right\rVert}
\newcommand{\vertexrep}[1]{\text{vertex}(#1)}
\newcommand{\BW}[1]{{\color{black} #1}}
\newcommand{\NEW}[1]{{\color{black} #1}}

\newbool{arxiv}
\setbool{arxiv}{true}
\newif\ifappendix
\appendixtrue

\def\NrObst{n_{o}}

\def\NrDOF{n_c}
\def\NrHorizon{H}
\def\NrAuxSteps{n_a}

\def\DimParamME{n_{p}}
\def\DimState{n_{x}}
\def\DimConf{\NrDOF}

\def\tubeGrowth{\tilde{\rho}}

\def\tubeSize{\delta}

\def\timeStep{T_{\text{s}}}
\def\tubeSizePred{\hat{\tubeSize}}
\def\xNomPred{\hat{\xNom}}
\def\c{\textbf{c}}
\def\e{\textbf{e}}

\def\zeroVec{\textbf{0}}
\def\eTrf{\tilde{\e}}
\def\eTrfStar{\tilde{\e}^{\star}}

\def\v{\textbf{v}}

\def\q{\textbf{q}}
\def\qdot{\dot{\textbf{q}}}
\def\qddot{\ddot{\textbf{q}}}
\def\qStart{\q_{\text{s}}}
\def\qGoal{\q_{\text{g}}}
\def\qGoalVirtual{\tilde{\q}_{\text{g}}}
\def\qc{\textbf{q}_{\text{c}}}

\def\grav{\textbf{g}}
\def\gravNom{\textbf{g}_0}
\def\gravErr{\textbf{g}_{\paramME}}
\def\gravDelta{\tilde{\grav}_{\paramME}}

\def\a{\textbf{a}}
\def\u{\textbf{u}}
\def\bubbleCenter{\textbf{c}}

\def\acc{\textbf{a}}
\def\accNom{\bar{\acc}}
\def\matAccNom{\bar{\textbf{A}}}
\def\setAcc{\mathcal{A}}

\def\b{\textbf{b}}

\def\bx{\textbf{b}_x}

\def\x{\textbf{x}}
\def\xNom{\bar{\x}}

\def\xGoal{\x_{\text{g}}}
\def\xGoalVirtual{\tilde{\x}_{\text{g}}}
\def\qGoalVirtual{\tilde{\c}_{\text{g}}}

\def\matA{\textbf{A}}

\def\matAx{\textbf{A}_x}
\def\matAu{\textbf{A}_u}

\def\matAccNom{\bar{\matA}}
\def\matAcl{\matA_{\text{cl}}}
\def\matB{\textbf{B}}
\def\matK{\textbf{K}}
\def\matP{\textbf{P}}

\def\matPinvSqrt{\matP^{-1/2}}

\def\matY{\textbf{Y}}
\def\matE{\textbf{E}}

\def\matV{\textbf{V}}
\def\matM{\textbf{M}}
\def\matC{\textbf{C}}

\def\matMnom{{\textbf{M}_0}}
\def\matMerr{\textbf{M}_{\paramME}}
\def\matMdelta{\tilde{\matM}_{\paramME}}

\def\matC{\textbf{C}}
\def\matCdelta{\tilde{\textbf{C}}_{\paramME}}
\def\matCnom{\textbf{C}_0}
\def\matCerr{\textbf{C}_{\paramME}}

\def\matQ{\textbf{Q}}
\def\matR{\textbf{R}}

\def\matK{\textbf{K}}

\def\matXnom{\bar{\textbf{X}}}

\def\w{\textbf{w}}
\def\wWC{\bar{w}}
\def\paramME{\bm{\theta}}

\def\setB{\mathcal{B}}

\def\setC{\mathcal{C}}
\def\setCdot{\mathcal{V}}
\def\setCf{\setC_{\text{f}}}
\def\setCo{\setC_{\text{o}}}
\def\setCoBnd{\partial\Co}
\def\setO{\mathcal{O}}

\def\setR{\mathbb{R}}
\def\setRconf{\setR^{\DimConf}}
\def\setRconfByconf{\setR^{\DimConf \times \DimConf}}

\def\setPD{\mathbb{S}_{++}}

\def\setW{\mathcal{W}}
\def\setU{\mathcal{U}}
\def\setX{\mathcal{X}}
\def\setN{\mathbb{N}}
\def\setV{\mathcal{V}}
\def\setFK{\mathcal{FK}}
\def\setD{\mathcal{D}}

\def\setCoBnd{\partial\setCo}
\def\setE{\mathcal{E}}
\def\setEcont{\mathcal{E}_{\acc}}
\def\setBQ{\mathcal{E}_{\q}}

\def\matI{\mathbf{I}}

\def\paramMEset{\Theta}

\def\tupleBalls{\bar{\text{B}}}
\def\tupleCorr{\text{B}}
\def\tuplePath{\text{C}}

\def\funcSD{\text{r}}
\def\funcPath{\bm{\gamma}}

\def\funcFL{\pi_{\text{FL}}}

\def\funcModelError{\Delta_{\paramME}}
\def\funcDiscError{\Delta_{\text{disc}}}
\def\funcModelErrorBound{\beta}

\title{\LARGE \bf Robust Convex Model Predictive Control with collision avoidance guarantees for robot manipulators}
\author{Bernhard Wullt, Johannes Köhler, \IEEEmembership{Member, IEEE}, Per Mattsson, \IEEEmembership{Member, IEEE}, Mikael Norrlöf and Thomas~B.~Schön,~\IEEEmembership{Senior~Member,~IEEE}
\thanks{This research was supported by the \emph{Wallenberg AI, Autonomous Systems and Software Program (WASP)} funded by Knut and Alice Wallenberg Foundation.}
\thanks{Bernhard Wullt and Mikael Norrlöf are with ABB robotics, 721 36 Västerås, Sweden (e-mail: bernhard.wullt@se.abb.com, mikael.norrlof@se.abb.com).}
\thanks{Johannes Köhler is with the Department of Mechanical Engineering, Imperial College London, London, UK, (e-mail: j.kohler@imperial.ac.uk).}
\thanks{Per Mattsson and Thomas B. Schön are with the Department of Information Technology, Uppsala University, 751 05 Uppsala, Sweden (e-mail: per.mattsson@it.uu.se, thomas.schon@uu.se). }
}
\begin{document}
\maketitle
\begin{abstract}
Industrial manipulators typically operate in cluttered environments, where safe motion planning is critical. However, model uncertainties further complicate this task, which leads to conservative speed limits to reduce the influence of disturbances. Hence, there is a need for control methods that can guarantee safe motions which are executed fast. We address this by suggesting a novel model predictive control (MPC) solution for manipulators, where our two main components are a robust tube MPC and a corridor planning algorithm to obtain collision-free motion. Our solution results in a convex MPC formulation, which we can solve fast, making our method practically useful. We demonstrate the efficacy of our method in a simulated environment with a 6 DOF industrial robot operating in cluttered environments with uncertain model parameters. We outperform benchmark methods by tolerating higher levels of model uncertainty while achieving faster motion. 
\end{abstract}
\thispagestyle{empty}
\pagestyle{empty}
\section{Introduction}
\NEW{Motion planning and control for robotic manipulators are fundamental components of an autonomous robotics pipeline, which typically includes a task planner, a motion planner, and a controller. The task planner decomposes a high-level objective into a sequence of sub tasks, increasingly leveraging Large Language Models \cite{llm3}, and invokes a motion planner to generate collision-free trajectories for execution. Motion planning is commonly done in a decoupled way \cite{lavallebook}, first, through path planning \cite{rrt}, which then is time scaled using the dynamics model \cite{topps_cvx}. Once the plan is executed, the resulting trajectory is tracked through feedback-linearization \cite{rob_mod_plan_cont} using the full dynamics model to follow it closely. This decoupled approach to planning and control is efficient. However, it relies on accurate modeling and is further limited to stay exactly on the collision-free path, since it is only there it has been certified to be collision-free. Naturally, model uncertainties are part of the problem, which for manipulators can have a wide spread effect due to the coupled dynamics, potentially  putting the robot offtrack, which may result in collisions. Practically, to attenuated its effect, the trajectory is tracked by moving sufficiently slow. Thus, a limiting factor is the lack of methods to guarantee robustness during faster motions.} We address the missing robustness guarantees, \NEW{by proposing an integrated motion planning and control solution, providing} safe and fast motion in cluttered environments. We realize this through two main contributions:
\begin{itemize}
    \item We use feedback linearization to obtain a linear prediction model and derive a state and input dependent upper bound on the resulting model error. We design a tube based model predictive controller (MPC) that utilizes this bound and  optimizes the tube size to reduce conservatism. 
    \item To propagate the tube in the collision-free part of the configuration space, we employ a signed configuration distance function (SCDF), which outputs collision-free balls in the configuration space. This allows us to formulate simple convex constraints for obstacle avoidance, while also propagating the tube in a collision-free region, enabling quick collision-free progress towards the goal.
\end{itemize}
\begin{figure}
\centering
\includegraphics[width=6cm, trim={0cm 4cm 0cm 1cm}, clip]{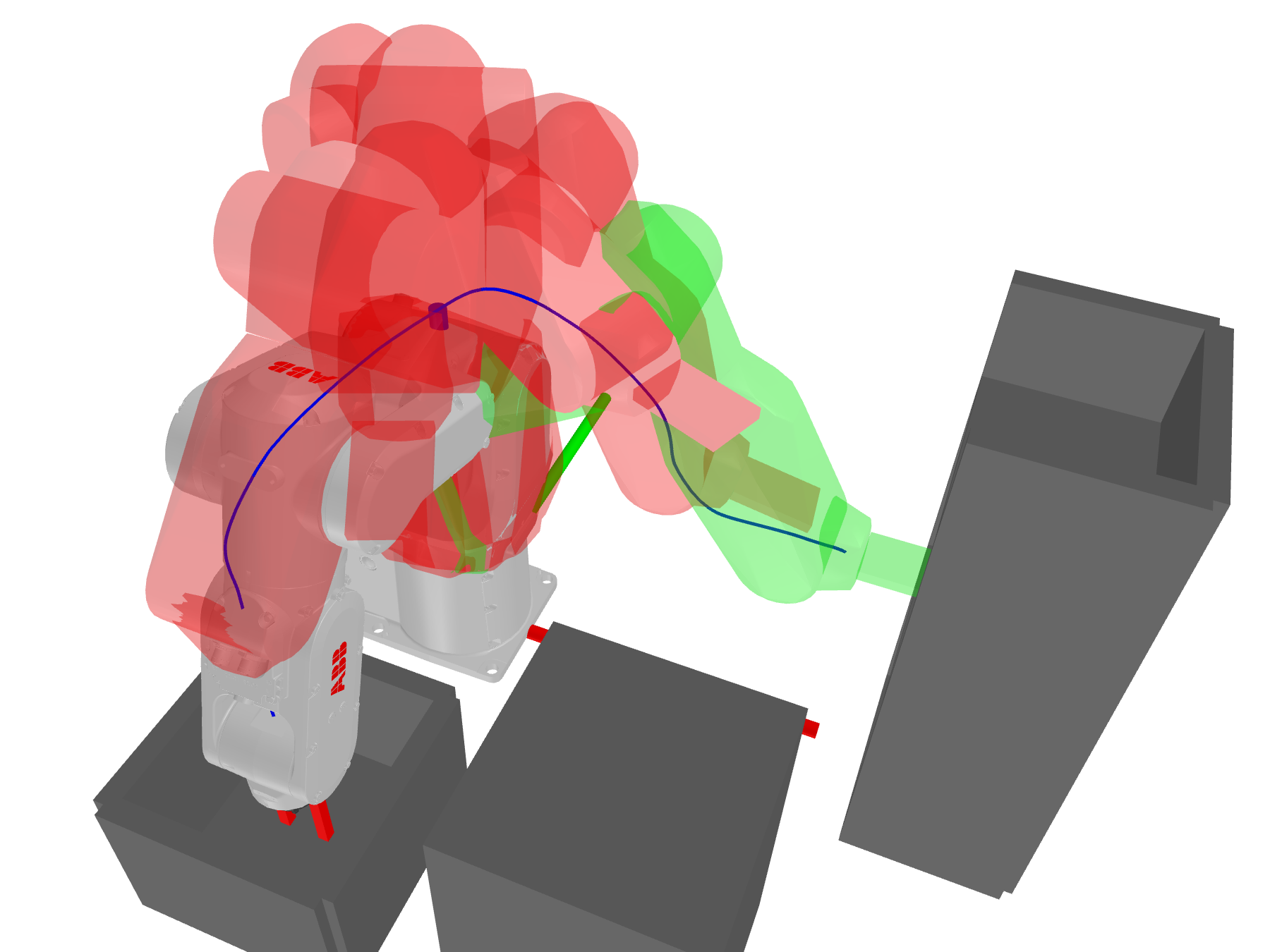}
\caption{Illustration of our convex robust MPC. We produce collision-free motion that is robust to model errors. The closed loop trajectory (transparent red meshes) connect a given start and goal state (white and green meshes) while avoiding the obstacles (gray spheres).}
\label{fig:demo}
\end{figure}
The result is a convex MPC problem, which we can solve efficiently, and a simple approach to guide it to the goal state, resulting in fast and safe motion. This is the first construction of a convex MPC solution that guarantees collision avoidance for nonlinear uncertain dynamics of robots manipulators. We demonstrate the practical applicability of our approach through simulations of a proprietary 6 DOF industrial robot (see Figure~\ref{fig:demo}). In addition, we provide an open-source implementation of the proposed method for general manipulators:
\begin{center}
{\url{https://github.com/berwul/rob_cvx_rmpc}}.    
\end{center}
\newpage
\subsubsection*{Outline} 
Section~\ref{sec:related_work} presents related work. We introduce the notation in Section~\ref{sec:notation} and the problem formulation in Section~\ref{sec:problem_formulation}. Next, we derive our novel robust motion planning solution for manipulators in Section \ref{sec:man_robust_control}. 
Obstacle avoidance is included in Section~\ref{sec:obs_avoid} by proposing the concept of corridor planning and deriving simple convex constraints ensuring collision-free motion. We verify our approach in numerical experiments (Section~\ref{sec:experiments}) and end with conclusions (Section~\ref{sec:conclusions}).
\section{Related work}
\label{sec:related_work}
Forming collision-free regions in the configuration space is challenging for manipulators, due to the non-trivial mapping of world space obstacles to the configuration space. In the past decade, a lot of methods \cite{iris_c, IRIS_nlp, cspf, pbrm} have been developed in order to produce convex collision-free regions in the configuration space, \cite{iris_c, IRIS_nlp} iteratively grows a collision-free ellipsoid through optimization, while~\cite{cspf, pbrm} uses learning to produce collision-free balls. These tools enable new approaches to motion planning, e.g. trajectory planning \cite{mpGcs}, path planning \cite{pbrm}, manipulation planning \cite{cspf}. Our work makes use of these new representation capabilities to formulate convex obstacle avoidance constraints and combines it with a novel robust control formulation to ensure safe and efficient motions.
\\\\
Our approach relies on MPC \cite{mpc_book}, which, \NEW{compared to other integrated motion planning and control methods, e.g., \cite{kdf}}, guarantees satisfaction of state and input constraints.
In particular, we build on MPC-for-tracking formulations~\cite{krupa2024model}, which can progressively reach far away targets by optimizing artificial references. Obstacle avoidance constraints can in principle also be directly added to such a formulation~\cite{set_point_tracking_mpc,safe_n_fast_rmpc}. 
However, especially for robot manipulators these constraints are highly non-linear and non-convex, thus increasing the computational demand. 
Similar to the proposed approach, the work in~\cite{ciao} use convex collision-free balls, however, the methodology is limited to single-body robots with known dynamics. Applications of MPC to robot manipulators are, for example, presented in~\cite{safe_n_fast_rmpc, ddmpc, smg_mpc}. However, these approaches are computationally expensive, can only cover simple collision avoidance constraints or none at all, or are conservative.  In particular, \cite{safe_n_fast_rmpc, ddmpc} formulate non-linear MPC schemes, which becomes computationally expensive. Only simplified collision avoidance constraints are treated in~\cite{safe_n_fast_rmpc}, while~\cite{ddmpc, smg_mpc} do not address obstacle avoidance at all. The robustness guarantees in~\cite{safe_n_fast_rmpc, smg_mpc} are based on a constant (worst-case) bound on the model-mismatch, which neglects the state/input dependent nature of modeling errors, resulting in significant conservatism, while~\cite{ddmpc} lacks robustness guarantees. \BW{In contrast, we construct a collision-free corridor in the configuration space through a SCDF~\cite{pbrm}, resulting in convex obstacle avoidance constraints. We address model error through a robust design, which is tube based. Compared to standard rigid tube MPC \cite{tubeMPC}, which uses a polytopic invariant sets to bound the state trajectories, we use scaled ellipsoids, resulting in a more scalable design and efficient formulation for online control. In particular, we exploit a state and input dependent bound on the model error, allowing us scale the tube in a flexible way, while also reducing conservatism. We end up with an MPC formulation that is convex, which we can solve fast, observing real-time capabilities in numerical experiments.}
\section{Notation}
\label{sec:notation}
The set of positive real numbers is denoted by $\setR_+$.
We denote the set of integers $a$ to $b$, i.e. $\{a, a+1, \hdots, b \}$ by $\setN_{a:b}$. The set of $n$ dimensional positive definite matrices is denoted by $\setPD^{n}$. For a vector $\x\in\mathbb{R}^n$, we denote the 2-norm and infinity-norm as $\norm{\x}=\sqrt{\x^\top\x}$ and $\norm{\x}_{\infty} = \underset{i}{\max} \: |\x_i|$. For a matrix $\matA\in \setR^{n \times m}$,  $\norm{\matA}$ is the induced matrix norm, i.e., the largest singular value of $\matA$. The weighted vector norm is defined as $\norm{\x}_\matA=\sqrt{\x^\top \matA \x}$. We denote the symmetric matrix square root of a positive semi-definite matrix $\matA$ as $\matA^{1/2}$. The vectors $\textbf{1}_{n}$ and $\zeroVec_{n}$ denotes an $n$ dimensional vector of ones and zeros, respectively. The identify matrix is denoted as $\matI$. The operation of stacking two column vectors is expressed as $(\textbf{a}, \textbf{b}) = [\textbf{a}^\top, \textbf{b}^\top]^\top$. Finally, the function $\text{diag}(\cdot) : \setR^n \mapsto \setR^{n\times n}$, maps an $n$ dimensional vector into a diagonal matrix.
\section{Problem formulation}
\label{sec:problem_formulation}
Consider a robot operating in the world space, $\setW \subset \setR^3 $. The exact description of the robot body is given by its configuration $\q \in \setC$, where $\setC \subset \setR^{\NrDOF}$ is the configuration space. The robot dynamics has the following form
\begin{equation}
    \label{eq:dyn}
    \u = \matM(\q) \qddot + \matC(\q, \qdot) \qdot + \grav(\q).
\end{equation}
In the above, $\qdot \in \setCdot $, $\u \in \setU$, denote the velocity and control input, constrained to lie in their corresponding sets $\setCdot \subset \setRconf$ and $\setU \subset \setRconf$. Furthermore, $\matM: \setRconf \mapsto \setRconfByconf$, $\matC : \setRconf \times \setRconf \mapsto \setRconfByconf$, $\grav : \setRconf \mapsto \setRconf$, denote the mass matrix, coupling matrix (Coriolis and damping), and gravity vector functions, respectively. We assume incomplete knowledge of the model parameters \eqref{eq:dyn}, separating the model into nominal (known) and uncertain terms:
\begin{align*}
\matM(\q) &= \matMnom(\q) + \matMerr(\q), \\
\matC(\q, \qdot) &= \matCnom(\q, \qdot) + \matCerr(\q, \qdot), \\
\grav(\q) &= \gravNom(\q) + \gravErr(\q),
\end{align*}
where a null index denotes nominal terms, and error terms are indexed by a parameter vector $\paramME \in \paramMEset$, where $\paramMEset \subset \setR^{\DimParamME}$ is a known parameter set and $\paramME$ is the unknown model parameter. We assume that the sets $\setC$, $\setCdot$ and $\setU$ are hyperboxes, and that measurements of $\q$ and $\qdot$ are readily available. The robot is surrounded by $\NrObst$ obstacles $\setO = \bigcup_{i=1}^{\NrObst} \setO_i$, where $\setO_i \subset \setW$. The set of points covered by the robot body in configuration $\q$ is expressed as $\setFK(\q) \subset \setW$. We define the free space as $\setCf = \{\q \in \setC \; | \; \setFK(\q) \cap \setO = \emptyset \}$ and the obstacle region as $\setCo = \setC \setminus \setCf $. Our goal is to design a controller that steers the robot from a given start configuration $\qStart$ to a goal configuration $\qGoal$, while ensuring that the resulting trajectory satisfies constraints on velocity, torque, and is collision-free, i.e. $\q(t) \in \setCf$, $\qdot(t) \in \setCdot$, $\u(t) \in \setU$, $\forall t \geq 0$, for any considered model parameters $\paramME\in\Theta$.
\section{Robust convex MPC for manipulators}
\label{sec:man_robust_control}
In this section, we will derive a motion planning solution that comes with robustness guarantees, for the moment ignoring obstacle avoidance, which we address in the subsequent section.
\\\\
A key feature we aim for in our design is that the resulting MPC problem can be solved fast. Hence, we want a convex formulation, which requires convex constraints and linear prediction models. First, we utilize feedback linearziation to obtain a linear model and a suitable bound on the un-cancelled non-linearities (Section \ref{sec:fbd_lin}). 
To ensure robustness, we predict a scaled tube around a nominal prediction such that it contains the true (unknown) system response. 
We use an auxiliary controller and an ellipsoid tube to derive an expression of the tube dynamics (Section \ref{sec:aux_cont_tb_dyn}). 
Finally, we present our suggested convex robust MPC in Section~\ref{sec:rmpc:opt_prob}, including the theoretical analysis.
\subsection{Feedback linearization and model error}
\label{sec:fbd_lin}
In order to obtain a convex optimization problem, we have to obtain a linear prediction model. We use feedback linearization (FL~\cite{rob_mod_plan_cont}) to realize this, i.e., we cancel the (known) non-linear terms in the dynamics using feedback
\begin{equation}
    \label{eq:fdb_lin_map}
    \u = \funcFL(\q, \qdot, \acc) = \matMnom(\q) \acc + \matCnom(\q, \qdot) \qdot + \gravNom(\q),
\end{equation}
\BW{
where $\acc \in \setR^{\DimConf}$ is the desired acceleration. 
To ensure that the resulting torque $\u$ satisfies the torque constraints $\setU$, 
we introduce a convex acceleration constraint set $\setAcc \subset \setR^{\DimConf}$, which satisfies 
\begin{equation}
    \label{eq:func_fl_cond}
    \funcFL(\q, \qdot, \a)\in\setU \quad \forall\: \q \in \setC,\: \qdot\in \setCdot, \: \a \in \setAcc.
\end{equation}
}
Using the feedback~\eqref{eq:fdb_lin_map} in the dynamics~\eqref{eq:dyn} yields
\begin{align}
    \qddot =& \: \acc + \funcModelError(\q, \qdot, \acc), \label{eq:cont_dyanamic} \\
    \label{eq:model_error_exp}
    \funcModelError(\q, \qdot, \acc) =& \: \matMdelta(\q) \acc + \matCdelta(\q, \qdot) \qdot + \gravDelta(\q),
\end{align}
where \BW{$\qddot \in \setR^{\DimConf}$ is the acceleration. The functions} $\matMdelta : \setRconf \mapsto \setR^{\DimConf \times \DimConf}$, $\matCdelta : \setRconf \mapsto \setR^{\DimConf \times \DimConf}$ and $\gravDelta : \setRconf \mapsto \setRconf$ are errors parameterized by the uncertain model parameter $\paramME\in\Theta$, see\ifbool{arxiv}{~Appendix \ref{sec:model_err_der}}{~\cite[App. A]{arxiv_version_of_paper}} for the derivation. \BW{We note that~\eqref{eq:model_error_exp} does not introduce any approximations, it simply states that the uncertain robot dynamics~\eqref{eq:dyn} with the feedback~\eqref{eq:func_fl_cond} are equivalent to a double integrator subject to an additional nonlinear perturbation $\funcModelError$, that depends on the uncertain model parameters $\paramME$. Integrating the model~\eqref{eq:model_error_exp} over the sampling time $\timeStep$ with a (piece-wise) constant desired acceleration $\acc$ yields
\begin{align}
\label{eq:dynamics_discretized}
\x(k+1) &= \matA \x(k) + \matB(\acc(k)  + \funcModelError(\x(k), \acc(k))) \nonumber \\
&+\funcDiscError(\x(k), \acc(k)), 
\end{align}
where, $k \in \setN$ is the discrete time index, $\x = (\q, \qdot) \in \setX := \setC \times \setCdot\subset \setR^{\DimState}$ is the state with  $\DimState=2\cdot \DimConf$. To simplify the notation, we access the configuration and velocity from the state by $\q(\x)$ and $\qdot(\x)$, respectively. $\matA \in \setR^{\DimState \times \DimState}$ and $\matB \in \setR^{\DimState \times \DimConf}$ denote the dynamics and control matrices, respectively. The error term $\Delta_{\mathrm{disc}}$ is an additional discretization error accounting for the fact that $\Delta_\theta$ is not constant over the sampling time.}
\subsection{Auxiliary controller and tube dynamics}
\label{sec:aux_cont_tb_dyn}
To attenuate the effects of the model errors over a prediction horizon, we use an auxiliary controller, with the following control law
\begin{equation}
    \label{eq:aux_control_law}
    \acc = \accNom + \matK (\x - \xNom).
\end{equation}
The matrix $\matK \in \setR^{\DimConf \times \DimState}$, is referred to as the gain matrix, $\accNom \in \setR^{\DimConf}$ and $\xNom \in \setR^{\DimState}$ are the reference control and states, respectively. We compute the gain matrix that satisfies the following requirement
\begin{equation}
    \label{eq:contraction_req}
    \norm{(\matA+\matB\matK) \x }_{\matP} \le \rho \norm{\x}_{\matP},\quad \forall \x\in\setR^{\DimState},
\end{equation}
where $\rho \in (0, 1)$ is a contraction rate and $\matP \in \setPD^{\DimState}$ is referred to as the Lyapunov matrix. We obtain $\matP$ and $\matK$ satisfying~\eqref{eq:contraction_req} by solving a convex optimization problem offline, see\ifbool{arxiv}{~Appendix~\ref{sec:cvx_control_syn}}{~\cite[App. B]{arxiv_version_of_paper}} for details. The Lyapunov matrix $\matP$ allows us to form an ellipsoid in \BW{the state space which, together with its projection on to the configuration, velocity and control input space, are 
\begin{subequations}
\label{eq:ellipse}
\begin{align}
    \setE(\tubeSize) =& \{ \x \in \setR^{\DimState}  \:|\: \norm{\x}_{\matP} \le \tubeSize \}, \\
    \setE_{\q}(\tubeSize)=&\{\q(\x) \in \setR^{\DimConf}\:|\: \x\in\setE(\tubeSize)\}, \\
    \setE_{\v}(\tubeSize)=&\{\qdot(\x) \in \setR^{\DimConf}\:|\: \x\in\setE(\tubeSize)\},\\
    \setEcont(\tubeSize)=&\{\a=\matK\x\in\setR^{\DimConf} \:|\: \x\in\setE(\tubeSize)\}.
\end{align}
\end{subequations}}
The size of the ellipsoid is controlled by the scaling $\tubeSize \in \setR_+$.
\\\\
\BW{The following proposition provides a state and input dependent bound on the prediction error.
\begin{proposition}
\label{prop:delta_beta}
For all \( \x = (\q, \qdot) \in \setX \),  \( \acc \in \setAcc \) and \( \paramME \in \paramMEset \), we have
\begin{align}
\label{eq:model_error_upper_bound}
&\norm{\matB\funcModelError(\x, \acc) +\funcDiscError(\x, \acc)}_{\matP}
\\
 \le &\funcModelErrorBound(\x, \acc) :=  a\norm{\acc}+b\norm{\qdot(\x)}+c. \nonumber
\end{align}
with $\funcModelError,\funcDiscError$ from~\eqref{eq:dynamics_discretized} and $a,b,c$ according to~\eqref{eq:me_a}-\eqref{eq:me_c}.
\end{proposition}
\begin{proof}
Using the triangle inequality and the property $|| \matA \x|| \le || \matA || || \x||$ \cite{horn_john_matrix} on the individual terms of expression \eqref{eq:model_error_exp}, we end up with the bound in \eqref{eq:model_error_upper_bound}, where $a, b, c \in \setR_+$ are computed according to
\begin{align}
    a &= \underset{\paramME \in \paramMEset, \q \in \setC }{\text{max}} \: ||\matP^{1/2} \matB \matMdelta(\q)||, \label{eq:me_a} \\
    b &= \underset{\paramME \in \paramMEset, (\q, \qdot) \in \setX}{\text{max}} \: ||\matP^{1/2} \matB \matCdelta(\q, \qdot) ||, \label{eq:me_b}\\
    c &= \underset{\paramME \in \paramMEset, (\q, \qdot) \in \setX, \acc \in \setAcc }{\text{max}}  \: ||\matP^{1/2} (\matB \gravDelta(\q) + \funcDiscError(\q, \qdot, \a))||.
    \label{eq:me_c}
\end{align}
The closed-form expressions of the functions introduced above are presented in\ifbool{arxiv}{~Appendix \ref{sec:model_err_der}}{~\cite[App. A]{arxiv_version_of_paper}}.
\end{proof}
}
This bound highlights that the error depends significantly on the velocity and acceleration, crucial knowledge that the controller will leverage for safe and efficient planning. 
\\\\
Having obtained a state and input dependent bound, we now focus on how to adjust the scaling such that the uncertain system~\eqref{eq:dynamics_discretized} remains inside the ellipsoid~\eqref{eq:ellipse} around the nominal prediction\BW{, which we present in the following proposition.
\begin{proposition}
\label{prop:tube_dynamics}
For any $\norm{\x-\xNom}_\matP\leq \delta$, $\paramME\in\Theta$, we have that $\norm{\x_+-\xNom_+}_\matP\leq \delta_+$ with $\x_+$ according to~\eqref{eq:dynamics_discretized}, $\xNom_+=\matA\xNom+\matB\accNom$, $\acc=\accNom+\matK(\x-\xNom)$, $\matV= [\textbf{0}_{\DimConf \times \DimConf}, \textbf{I}_{\DimConf \times \DimConf}]$ and 
\begin{align}
\tubeSize_+ = & \: (\rho + L_\beta) \tubeSize + \beta(\xNom,\accNom), \label{eq:delta_dynamics}\\
L_\beta = & \: a\norm{\matK\matP^{-1/2}}+b\norm{\matV \matP^{-1/2}}. \label{eq:L_beta}
\end{align}
\end{proposition}
\ifbool{arxiv}{The proof can be found in Appendix~\ref{sec:proof_cl_error_prop}.}{The bound follows from \eqref{eq:model_error_upper_bound} and \eqref{eq:contraction_req}, see \cite[App. C]{arxiv_version_of_paper} for details.}
}In the following, we abbreviate $\tubeGrowth :=\rho+L_\beta$ and we assume that $\tubeGrowth <1$. 
Given $\rho<1$, this holds if the parametric uncertainty $\Theta$ is sufficiently small. Finally, for a steady state with zero velocity/acceleration, the tube size $\tubeSize$ converges to a steady state tube size
\begin{align}
\label{eq:delta_f}
\BW{\tubeSize_f:= c/(1-\tilde{\rho})}.
\end{align}
\BW{We assume that this steady-state tube size is smaller than the velocity and acceleration constraints:}
\begin{align}
\label{eq:origin_in_tightened_constraints}
\BW{\zeroVec_{\DimConf} \in\setAcc\ominus\setE_{\acc}(\delta_f+\epsilon),
\quad \zeroVec_{\DimConf} \in \setV \ominus \setE_{\v}(\delta_f+\epsilon).}
\end{align}
\subsection{Robust convex MPC problem}
\label{sec:rmpc:opt_prob}
With a convex expression for the model error propagation and convex input constraints established, we can now formulate our resulting robust MPC problem as follows:
\begin{subequations}
\label{eq:mpc}
\begin{align}
\underset{\matXnom, \matAccNom, \bm{\tubeSize} }{\operatorname{min}} 
    \quad 
    & 
    \sum_{i=0}^{\NrHorizon-1}\left[\norm{\xNom_i-\xNom_{\NrHorizon}}_{\matQ}^2+\norm{\accNom_i}_{\matR}^2\right]  +\norm{\xNom_{\NrHorizon}-\xGoal}_{\matQ_e}^2 \label{eq:htmpc:obj} \\
    \textrm{s.t.} \quad & \norm{\xNom_0 - \x(k)}_{\matP} \le \tubeSize_0 , &  \label{eq:htmpc:start} \\
        & \xNom_{i+1} = \matA \xNom_i + \matB \accNom_i,  \label{eq:htmpc:dyn} \\
        & \qdot(\xNom_{\NrHorizon}) = \mathbf{0}_{\DimConf},\tubeSize_\NrHorizon\geq \tubeSize_f,\xNom_\NrHorizon\in \setX\ominus\setE(\delta_\NrHorizon+\epsilon) \label{eq:htmpc:ss} \\
        & \tubeSize_{i+1} \ge \tubeGrowth \tubeSize_i + \funcModelErrorBound(\xNom_i, \accNom_i), \label{eq:htmpc:tube_dyn} \\
        & \xNom_i \in \setX \ominus \setE(\delta_i), & \label{eq:htmpc:limits_states} \\
        & \accNom_i \in \setAcc \ominus \setEcont(\delta_i), \; i \in \setN_{0:\NrHorizon-1}, & \label{eq:htmpc:limits_controls}
\end{align}
\end{subequations}
using the current state $\x(k)$, the goal state $\xGoal$, user chosen positive definite weight matrices $\matQ,\matQ_e,\matR \in \setPD$, and a horizon $\NrHorizon\geq 2$. \BW{The user-chosen offset $\epsilon>0$ ensures that the system cannot get stuck at a steady-state close to the constraints~\cite{krupa2024model}.} The decision variables are the nominal trajectory, $\matXnom = [\xNom_0, \hdots, \xNom_{\NrHorizon}] \in \setR^{\DimState \times (\NrHorizon+1)}$, the nominal control inputs $\matAccNom = [\accNom_0, \hdots, \accNom_{\NrHorizon-1}]\in \setR^{\DimConf \times \NrHorizon}$ and the tube size $\bm{\tubeSize} = [\tubeSize_0, \hdots, \tubeSize_{\NrHorizon}]^\top\in \setR^{\NrHorizon+1}$. In closed-loop operation, we solve~\eqref{eq:mpc} with the current measured state $\x(k)$, and then apply $\a(k)=\a^\star_0$, the first optimized input, to the system. Note that Problem~\eqref{eq:mpc} is a convex second-order cone problem, which can be efficiently solved. 
\\\\
The objective, \eqref{eq:htmpc:obj}, consists of \BW{two} parts. The first part drives the predicted trajectory to the steady-state $\xNom_\NrHorizon$, which acts as an artificial reference (cf.~\cite{krupa2024model}), while keeping the controls small. A second term pushes this artificial reference to the desired goal $\xGoal$. 
\\\\
The nominal trajectory starts with a tube around the measured state $\x(k)$~\eqref{eq:htmpc:start} and evolves according to the linear dynamics~\eqref{eq:htmpc:dyn}. The final state is a steady-state with zero velocity~\eqref{eq:htmpc:ss}. The  tube~\eqref{eq:htmpc:tube_dyn} is scaled according to Proposition~\ref{prop:tube_dynamics}. Finally, the state and control limits are tightened by the tube size in \eqref{eq:htmpc:limits_states} and \eqref{eq:htmpc:limits_controls}, respectively. The operator $\ominus$ denotes the Pontryagin differences, see\ifbool{arxiv}{~Appendix~\ref{sec:derive_opt_prob}}{~\cite[App. B2]{arxiv_version_of_paper}} for implementation details. The following theorem summarizes the closed-loop properties.
\begin{theorem}
\label{thm:robust_MPC}
Consider the nonlinear system~\eqref{eq:dynamics_discretized} with $\paramME\in\Theta$ and \BW{$\xGoal\in\setX \ominus \setE(\epsilon+\tubeSize_f)$}. Suppose that Problem~\eqref{eq:mpc} is feasible at $k=0$, then the closed-loop system satisfies:
\begin{itemize}
    \item Recursive feasibility: 
Problem~\eqref{eq:mpc} is feasible $\forall k\in\mathbb{N}$;
    \item Constraint satisfaction: $\x(k)\in\setX$, $\acc(k)\in\setAcc$, $\forall k\in\mathbb{N}$;
    \item \BW{Convergence:  $\underset{k\rightarrow\infty}{\lim\sup} ~\{\x(k)-\xGoal\}\subseteq\setE(\delta_f)$.}
\end{itemize}
\end{theorem}
\begin{proof}
The proof merges concepts from robust MPC using homothetic tube~\cite{sasfi2023robust} and MPC for tracking~\cite{krupa2024model,kohler2024analysis}.\\
\textbf{Part I: }
Given the optimal solution to problem~\eqref{eq:mpc} at time $k$, we consider the following feasible candidate solution
$\matXnom = [\xNom^\star_1, \hdots, \xNom_{\NrHorizon}^\star,\xNom_{\NrHorizon}^\star]$, $\matAccNom = [\accNom_1^\star, \hdots, \accNom_{\NrHorizon-1}^\star, \textbf{0}_{\DimConf}]$,  $\bm{\tubeSize} = [\tubeSize_1^\star, \hdots, \tubeSize_{\NrHorizon}^\star, \tubeSize_{\NrHorizon}]$. 
Here,  $\delta_{\NrHorizon}=\delta^\star_{\NrHorizon}\tilde{\rho}+ c$ according to~\eqref{eq:htmpc:tube_dyn} with $\beta(\xNom_{H},0)=c$, given that $\xNom_H^\star$ is a steady-state using~\eqref{eq:htmpc:ss}.  
Furthermore, $\tilde{\rho}<1$ and $\delta_H^\star\geq \delta_f$~\eqref{eq:delta_f} ensure that $\tubeSize_\NrHorizon\leq\tubeSize_\NrHorizon^\star$ and thus this appended solution also satisfies the tightened constraints~\eqref{eq:htmpc:tube_dyn} at $i=\NrHorizon$.
Lastly, Proposition~\ref{prop:tube_dynamics} ensures that this candidate solution also satisfies the initial state constraint~\eqref{eq:htmpc:start} with the new measured state $\x(k+1)$ for any $\paramME\in\Theta$.\\
\textbf{Part II: }Closed-loop constraint satisfaction follows from the feasibility of Problem~\eqref{eq:mpc}, the tightened constraints~\eqref{eq:htmpc:limits_states}, \eqref{eq:htmpc:limits_controls}, and the fact that $\x(k)-\xNom_0^\star\in\setE(\delta_0^\star)$, $\acc(k)-\accNom_0^\star\in\setEcont(\delta_0^\star)$ using the definition of the ellipsoids~\eqref{eq:ellipse} and the initial state constraint~\eqref{eq:htmpc:start}.\\
\textbf{Part III: }Let us denote the optimal cost of Problem~\eqref{eq:mpc} at time $k$ by $\mathcal{J}^\star(k)$. 
The feasible candidate solution implies
\begin{align*}
&\mathcal{J}^\star(k+1)-\mathcal{J}^\star(k) 
\leq -\norm{\xNom_0^\star-\xNom_{\NrHorizon}^\star}_{\matQ}^2-   \norm{\accNom_0^\star}_{\matR}^2.
\end{align*}
Given $\mathcal{J}^\star$ non-negative and $\mathcal{J}^\star_0$ finite, using this condition in a telescopic sum ensures that, as $k\rightarrow\infty$, $\xNom^\star_0$ converges to a steady-state. 
Lastly, to ensure that this steady-state corresponds to $\xGoal$, we use~\cite[Lemma~1]{kohler2024analysis}, to ensure the existence of a uniform constant $d>0$, such that
\begin{align*}
\norm{\xNom_0^\star-\xNom_\NrHorizon^\star}^2_\matQ\leq d \norm{\xNom_\NrHorizon-\xGoal}_{\matQ_e}^2.
\end{align*}
This result applies given the convex steady-state manifold, the strictly convex quadratic cost, the fact that steady-states are the interior of the (tightened) constraints with $\epsilon>0$, 
and controllability of $(\matA,\matB)$ with $\NrHorizon\geq 2$. 
\BW{Thus, $\xNom^\star_0$ converges to $\xGoal$ and $\x(k)$ converges to $\xGoal\oplus\setE(\delta_f)$.}
\end{proof}
Notably, Theorem~\ref{thm:robust_MPC} relies only on convex optimization, provides robustness guarantees for uncertain nonlinear manipulators, accounts for velocity/acceleration dependence of the model error, and provides a larger region of attraction.\footnote{%
\BW{Given~\eqref{eq:origin_in_tightened_constraints}, any steady-state $\x(0)\in\setX\ominus\setE(\epsilon+\delta_f)$ is a feasible initial condition of the MPC~\eqref{eq:mpc}. 
Furthermore, in case the uncertainty in gravity and the discretization error is small, we get $c\approx 0, \delta_f\approx 0$. With $\epsilon\approx 0$ this implies that any steady-state in the constraints is a feasible initial condition.}
}
\newpage
\section{Real-time MPC with convex obstacle avoidance constraints}
\label{sec:obs_avoid}
\begin{figure*}[t!]
\begin{tikzpicture}
\node[inner sep=0pt] (corr) at (0,0) {
\includegraphics[width=.25\textwidth]{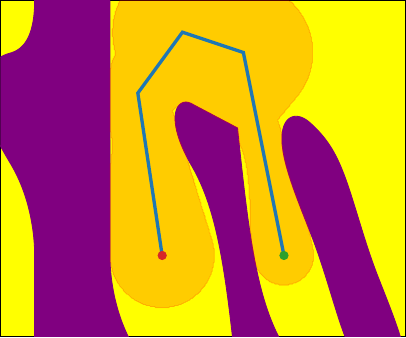}
};
\node[inner sep=0pt] at (4.5, 0) {
\includegraphics[width=.25\textwidth]{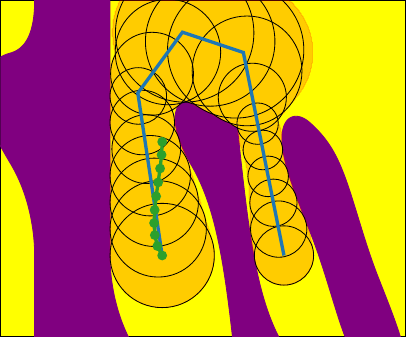}
};
\node[inner sep=0pt] at (9, 0) {
\includegraphics[width=.25\textwidth]{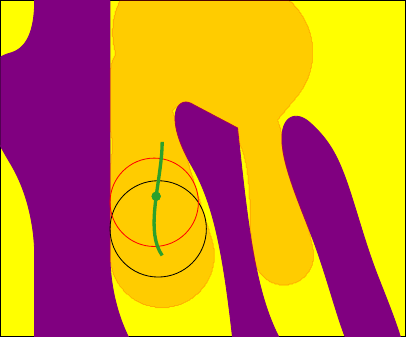}
};
\node[inner sep=0pt] at (13.5, 0) {
\includegraphics[width=.25\textwidth]{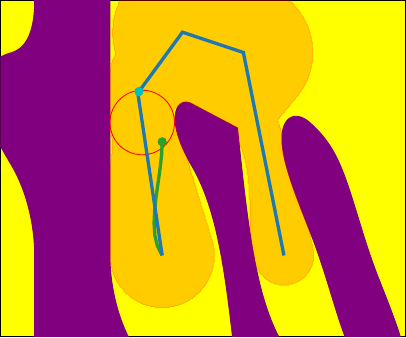}
};
\end{tikzpicture}
\caption{All plots are illustrated in the configuration space. The yellow and purple symbolize the free and the obstacle regions, respectively. \textbf{Left:} From a collision-free path (blue line), a corridor is defined by the SCDF (orange region). \textbf{Mid left:} The corridor is discretized, resulting in a sequence of balls (black circles). The idea is to select a set of collision-free balls and a temporary goal state, such that the predicted trajectory makes progress to the goal and stays within the corridor. A feasible trajectory from the previous iteration is used as help (dotted green curve). \textbf{Mid right:} The first step is to compute the balls. For each state along a given feasible trajectory, a ball is allocated. An example state is highlighted (green point). The balls that contain the state are selected (black and red circles). Then the ball with the largest margin to the point is selected (red circle). \textbf{Right:} Illustration of how the virtual goal is computed. The ball corresponding to the last point along the trajectory is selected (red circle). Then, a virtual goal (cyan point) is computed by finding a point along the path, contained inside the ball, with the furthest progress along the path.}
\label{fig:corr}
\end{figure*}
\BW{In this section, we add the two missing pieces of the approach presented in Section~\ref{sec:man_robust_control}: (i) How to enforce obstacle avoidance? (ii) How to execute the optimization in parallel to ensure real-time applicability? Furthermore, we present an algorithm that guides the robot through a collision-free corridor, only requiring the solution to a convex MPC problem that is solved in parallel during execution. Finally, we provide a theoretical analysis, showing that this approach robustly ensures safe operation.
}
\subsection{Obstacle avoidance through SCDF}
\label{sec:obs_avoid:scdf}
The SCDF is defined as 
\begin{equation}
\label{eq:SCDF}
\funcSD(\q) = 
\begin{cases}
-\min_{\qc \in \setCoBnd} \: \|\q - \qc\| & \text{if } \q \in \setCo, \\
\hphantom{-} \min_{\qc \in \setCoBnd} \: \|\q - \qc\| & \text{otherwise},
\end{cases}
\end{equation}
which is the distance to the boundary of the obstacle region $\setCoBnd$. 
Given a collision-free point $\bubbleCenter\in\setCf$, the SCDF allows us to define a collision-free region as 
\begin{equation}
\label{eq:cball}
\setB(\bubbleCenter) = \: \{\q \in \setRconf \: | \: \norm{\q - \bubbleCenter } \le \funcSD(\bubbleCenter) \: \} \subseteq \setCf,
\end{equation}
which is parametrized by the tuple $(\bubbleCenter, r)$. The region is an Euclidean norm ball in the configuration space. Obtaining an analytical expression for the SCDF is non-trivial, which is why we resort to approximate methods\cite{pbrm, cspf}. To adapt the proposed MPC such that it guarantees obstacle avoidance, we simply add the following constraints
\begin{gather}
\begin{aligned}
    \label{eq:ball_constraint}
    \q(\xNom_i) &\in \setB_i \ominus \setBQ(\delta_i), \; i \in \setN_{0:\NrHorizon-1},\\
    \q(\xNom_\NrHorizon) &\in \setB_\NrHorizon \ominus \setBQ(\delta_\NrHorizon+\epsilon),
\end{aligned}
\end{gather}
to the MPC~problem~in~\eqref{eq:mpc}. The computation of the constraint tightening is given in\ifbool{arxiv}{~Appendix~\ref{sec:tube_in_ball_const}}{~\cite[App. D]{arxiv_version_of_paper}}. In the above, each state in the nominal trajectory is constrained to lie within an allocated ball $\setB_i$. How these are computed is presented in Section \ref{sec:corr_planning}.
\BW{
\subsection{Parallel MPC formulation}
\label{sec:parallel_execution}
To fulfill the real-time constraints, we solve the MPC problem \eqref{eq:mpc} in parallel while running the auxiliary controller~\eqref{eq:aux_control_law} for $\NrAuxSteps > 1$ steps. To enable this parallelism, we need to use a forward projection of the initial condition~\cite{findeisen2004computational}, since the measured state $\x(k)$ in \eqref{eq:htmpc:start} is not yet available. Thus, we replace the measured state in the constraint~\eqref{eq:htmpc:start} with a tube prediction which will contain the future state:
\begin{equation}
\label{eq:real_time_initial_tube}
\norm{\xNom_0 - \hat{\xNom}_{\NrAuxSteps}}_{\matP} \le \tubeSize_0 - \hat{\tubeSize}_{\NrAuxSteps}.
\end{equation}
Here, $\hat{\xNom}_{\NrAuxSteps} \in \setX$ and $\hat{\tubeSize}_{\NrAuxSteps} \in \setR_+$ are the nominal state and tube size predicted at time $k-\NrAuxSteps$. 
In particular, $\hat{\xNom}_{\NrAuxSteps}$ corresponds to the prediction $\xNom^\star_{\NrAuxSteps}$ made $\NrAuxSteps$ steps ago.
Similarly, the tube is predicted using~\eqref{eq:htmpc:tube_dyn}, but with the initial condition based on the most up-to-date information at time $k-\NrAuxSteps$:}
\begin{equation}
\label{eq:predict_tube_size}
\BW{\hat{\tubeSize}_0 = \norm{\hat{\xNom}_{0} - \x(k-\NrAuxSteps)}_{\matP}.}
\end{equation}
\subsection{Robust corridor planning}
\label{sec:corr_planning}
Next, we propose an algorithm that ensures that the robot reaches a given goal configuration, $\qGoal$, without collisions. The high-level idea is to first generate a collision-free region that connects the current configuration and the goal configuration. We refer to this region as a corridor, which is produced using the SCDF around a collision-free path and is illustrated in the left part of Figure \ref{fig:corr}. Then, \BW{at each iteration,} we constrain the predicted MPC trajectory to stay within this corridor and pull it towards a temporary virtual goal state, $\xGoalVirtual$, such that we make progress to the global goal state.
\\\\
Our approach is sketched in Algorithm~\ref{alg:rbmp}. Before any planning can take place, we perform the offline tasks already described in the previous sections. That is, we compute the acceleration constraints $\setAcc$, line \ref{alg:rbmp:off_convexify}, compute the auxiliary controller, line \ref{alg:rbmp:off_solve_cvxp}\BW{, and compute the model error constants, line \ref{alg:rbmp:off_comp_params}.}
\begin{algorithm}[t!]
\centering
\caption{Our corridor planning approach to control the manipulator to a goal state while avoiding obstacles.}
\label{alg:rbmp}
\begin{algorithmic}[1]
	\State learn SCDF \Comment{Offline}
    \BW{
    \State $\setAcc \gets$ compute acceleration constraint \eqref{eq:func_fl_cond}\label{alg:rbmp:off_convexify}
    \State $\matP, \matK \gets$ solve optimization problem, see\ifbool{arxiv}{~Appendix \ref{sec:cvx_control_syn}}
    {~\cite[App. B]{arxiv_version_of_paper}}\label{alg:rbmp:off_solve_cvxp}
    \State $a, b, c \gets$ compute constants \eqref{eq:me_a}-\eqref{eq:me_c}\label{alg:rbmp:off_comp_params}
    }
    \Statex\hrulefill
    \State $\qGoal \gets$ get desired goal \Comment{Corridor planning}  \label{alg:rbmp:query}
    \State $k$ $\gets$ 0,  \BW{$\x(0)=(\qStart, \zeroVec_{\DimConf})$ $\gets $ initial condition at rest} \label{alg:rbmp:assign_x0}
    \State $\funcPath \gets $ run path planner \label{alg:rbmp:run_pp}
    \State $\tuplePath$, $\tupleCorr$  $\gets$ discretize $\funcPath$, compute SCDF and verify \eqref{eq:condition_corridor}\label{alg:rbmp:disc}
    \State $\matXnom$ $\gets$ $[\x(0)]_{i=0}^{\NrHorizon}\BW{, \matAccNom \gets [\zeroVec_{\DimConf}]_{i=1}^{\NrHorizon}, \xNomPred \gets \x(0),  \tubeSizePred \gets 0}$ \label{alg:rbmp:feasb_start}
    \Statex \hrulefill
    \State $\tupleBalls$ $\gets$ assign balls with $\matXnom$ \eqref{eq:assign_ball_rule} \Comment{Real-time control}\label{alg:rbmp:assign_ball}
    \State $\xGoalVirtual$ $\gets$ select virtual goal \eqref{eq:assign_vgoal} with $\tuplePath$ and $\setB_\NrHorizon$ \label{alg:rbmp:assign_vg}
    \State $\matXnom, \matAccNom$ $\gets$ \BW{solve in parallel} \eqref{eq:mpc}, \eqref{eq:ball_constraint}, \BW{\eqref{eq:real_time_initial_tube}}, with $\BW{\xNomPred, \tubeSizePred}, \xGoalVirtual, \tupleBalls$ \label{alg:rbmp:mpc}
    \BW{\State $\xNomPred$, $\tubeSizePred$ $\gets$ from $\x(k)$, predict future tube, Section \ref{sec:parallel_execution} }\label{alg:rbmp:predict_tube}
    \For{$i \in \setN_{0:\NrAuxSteps-1} $} \label{alg:rbmp:aux_start} \Comment{aux. control loop}
        \State $\accNom, \xNom \gets$ get $i$-th reference state and control
        \State apply $\u=\funcFL(\q, \qdot, \a)$, with $\a$ from \eqref{eq:aux_control_law} duration $\timeStep$ \label{alg:rbmp:apply}
        \State $\x(k+1) \gets $ measure next state 
        \State $k \gets  k +1$ 
    \EndFor \label{alg:rbmp:aux_end}
    \State $\matXnom$ $\gets$ shift $\NrAuxSteps$ times and append $[\xNom_\NrHorizon]_{i=1}^{\NrAuxSteps}$ \label{alg:rbmp:shift_repeat}
    \State Go back to \ref{alg:rbmp:assign_ball} \label{alg:rbmp:end}
\end{algorithmic}
\end{algorithm}
\BW{During online operation, we start at a steady-state configuration $\qStart$, receive a goal configuration $\qGoal$ and first execute the corridor planning, lines \ref{alg:rbmp:query}-\ref{alg:rbmp:assign_x0}.}
Then, we plan a high-level collision-free path, $\funcPath(s):~[0, 1]~\mapsto~\setCf$, to the goal, e.g. by using a sampling-based planner, line~\ref{alg:rbmp:run_pp}. Next, in line~\ref{alg:rbmp:disc}, we discretize the corresponding path, resulting in the sequence $\tuplePath =(\c_1, \hdots, \c_M) \in \setCf^{M}$, and precompute the collision-free balls with the SCDF according to \eqref{eq:cball}, resulting in a discretized corridor $\tupleCorr=((\c_1, r_1), \hdots, (\c_M, r_M))$. \BW{
In order to guarantee convergence, the sequence of balls needs to be sufficiently overlapping, which is captured by the following condition
\begin{equation}
\label{eq:condition_corridor}
 \c_{i+1} \in \setB(\c_{i}) \ominus \setE_\q(2\epsilon + \tubeSize_f).
\end{equation}
We verify~\eqref{eq:condition_corridor} for all neighbouring balls in the corridor and use a finer discretization for pairs where it is not yet fulfilled. 
Then, we initialize a feasible trajectory, control inputs, and a predicted tube around the current state, line~\ref{alg:rbmp:feasb_start}}. Next, we enter the real-time control loop, line \ref{alg:rbmp:assign_ball}-\ref{alg:rbmp:end}, which is illustrated in the mid-left part of Figure~\ref{fig:corr}. From a feasible trajectory we loop over all its states $\xNom_i \in \matXnom $ and assign a corresponding ball to it, line \ref{alg:rbmp:assign_ball}. This is done according to
\begin{equation}
    \label{eq:assign_ball_rule}
    (\c_i,r_i) = \arg\max_{(\c_j, r_j) \in \tupleCorr } \: \{ 
    r_j - \norm{\q(\xNom_i) - \c_j} \},
\end{equation}
which returns the ball $\setB_i=\setB(\c_i)$ that contains $\q(\xNom_i)$ with the largest margin. The assignment rule is conceptualized in the mid right part of Figure \ref{fig:corr}. The process is repeated for $i\in \setN_{0:\NrHorizon}$, resulting in the sequence $\tupleBalls = (\setB_0, \hdots, \setB_\NrHorizon)$.
\\\\
Next, we compute a virtual goal state $\xGoalVirtual$, line \ref{alg:rbmp:assign_vg}, which serves the purpose of pulling the trajectory in a direction that makes progress to the global goal state $\xGoal$. This is done by selecting the configuration that has made the most progress along the path and is contained in the last ball, i.e. according to
\begin{equation}
    \label{eq:assign_vgoal}
    \qGoalVirtual
    =
    \arg\max_{\c_i \in \tuplePath} \:  \{ i\mid
    \c_i \in \setB_{\NrHorizon} \ominus \setE_\q(\epsilon + \tubeSize_f)
    \},
\end{equation}
where the last ball is tightened with the steady state tube size in order to be compliant with the convergence properties of Theorem \ref{thm:robust_MPC}. Having computed a virtual goal configuration, we define the resulting state as $\xGoalVirtual= (\qGoalVirtual, \textbf{0}_{\DimConf})$. We illustrate this process in the right part of Figure~\ref{fig:corr}. \BW{We continue with solving the MPC problem in parallel, line \ref{alg:rbmp:mpc}, where we solve Problem~\eqref{eq:mpc} with the collision-avoidance constraints~\eqref{eq:ball_constraint} and initial tube constraint~\eqref{eq:real_time_initial_tube}. The inputs are the predicted tube, $\xNomPred$ and $\tubeSizePred$, balls $\tupleBalls$ and the virtual goal $\xGoalVirtual$}. Solving the problem results in an optimized nominal trajectory and control inputs. \BW{While the MPC solver is running, we continue by predicting our future tube, line \ref{alg:rbmp:predict_tube}, and run our auxiliary controller, line \ref{alg:rbmp:aux_start}-\ref{alg:rbmp:aux_end}. Next, we shift our nominal trajectory $\NrAuxSteps$ times and append the last state $\NrAuxSteps$ times, line \ref{alg:rbmp:shift_repeat}, serving as a feasible trajectory for the next iteration, line \ref{alg:rbmp:end}. The number of auxiliary steps defines a tunable time slot, $\NrAuxSteps \timeStep$, which enables parallel execution of the auxiliary controller and the MPC solver. It is defined by the user and helps to meet any real-time requirements.
\\
The following theorem shows that running our proposed algorithm guarantees feasibility and convergence to the goal state.}
\BW{\begin{theorem}
\label{thm:robust_MPC_coll_avoid}
Consider the nonlinear system~\eqref{eq:dynamics_discretized} with $\paramME\in\Theta$ and Algorithm~\ref{alg:rbmp}. 
Suppose further that the corridor reaches from $\qStart$ to the target $\qGoal$.
Then, 
\begin{itemize}
    \item Feasibility: All the optimization problems in Algorithm~\ref{alg:rbmp} are feasible for all $k\in\mathbb{N}$;
    \item Constraint satisfaction:
$\q(k) \in \setCf$, $\qdot(k) \in \setCdot$, $\u(k) \in \setU$, $\forall k \in\mathbb{N}$;
    \item Convergence: $\underset{k\rightarrow\infty}{\lim\sup} ~\{\x(k)-\xGoal\}\subseteq\setE(\tubeSize_f)$.
\end{itemize}
\end{theorem}
\begin{proof}
The proof is analogous to Theorem~\ref{thm:robust_MPC}. 
The main change is ensuring that the allocated ball constraints~\eqref{eq:ball_constraint} do not adversely impact convergence and feasibility and that the virtual goal $\xGoalVirtual$ converges to the global goal $\xGoal$ in finite time.
The detailed proof can be found in 
\ifbool{arxiv}{Appendix~\ref{app:proof_mpc_complete}.}{~\cite[App.~E]{arxiv_version_of_paper}.}
\end{proof}
}
\section{Numerical experiments}
\label{sec:experiments}
\NEW{The following section presents numerical experiments with the focus of understanding how our planner's performance is affected by increasing levels of uncertainty. Additional experiments are presented in \ifbool{arxiv}{Appendix~\ref{sec:dd}}{~\cite[App.~F]{arxiv_version_of_paper}}, where we instead derive the controller and model error bounds from real-world data, demonstrating that the proposed method remains effective under realistic uncertainty levels.

}
\subsection{Setup}
\begin{figure*}[t!]
\begin{tikzpicture}
\node[inner sep=0pt] at (0,0) {\includegraphics[width=.33\textwidth]{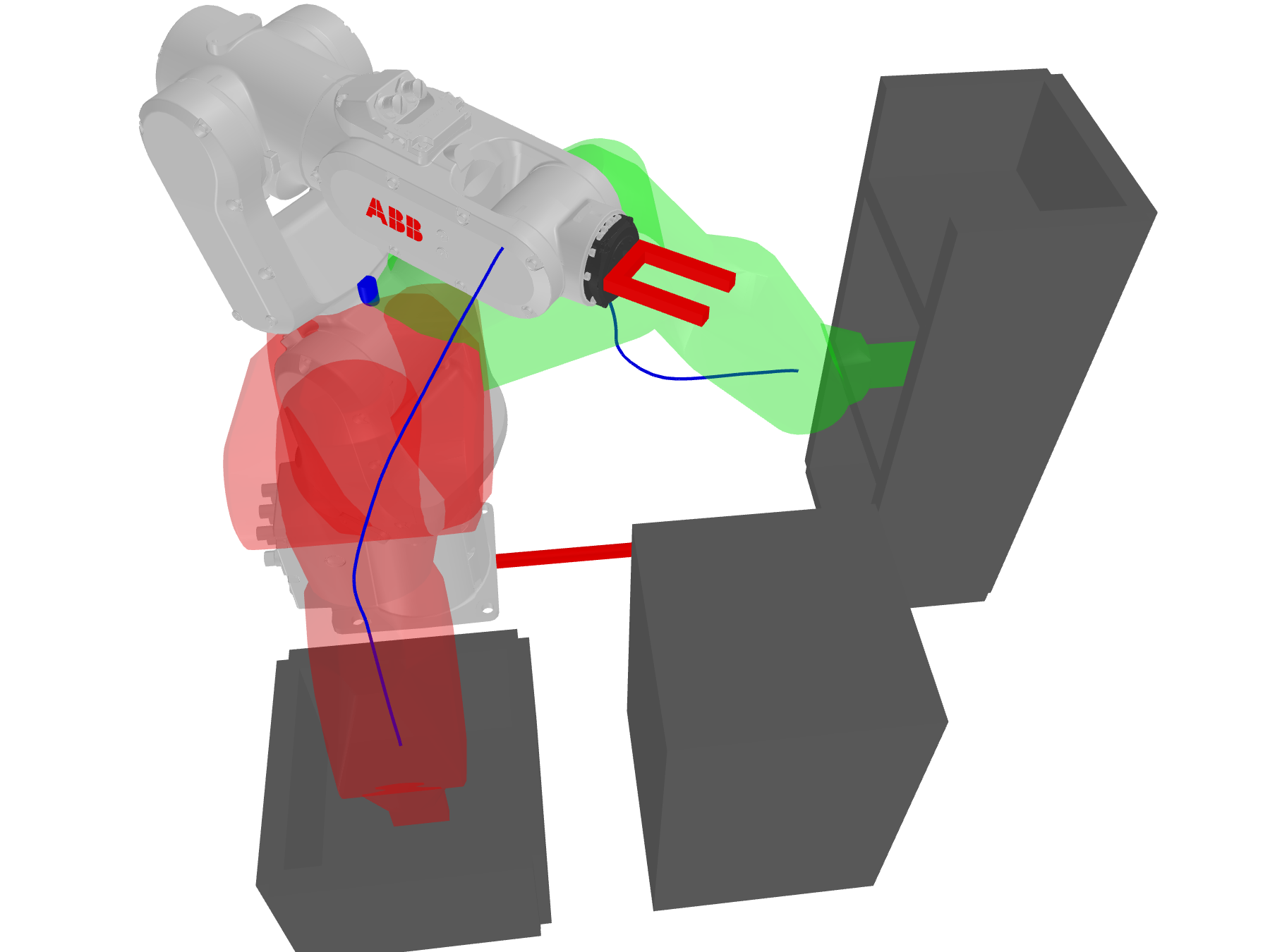}};
\node[inner sep=0pt] at (5, 0) {\includegraphics[width=.33\textwidth]{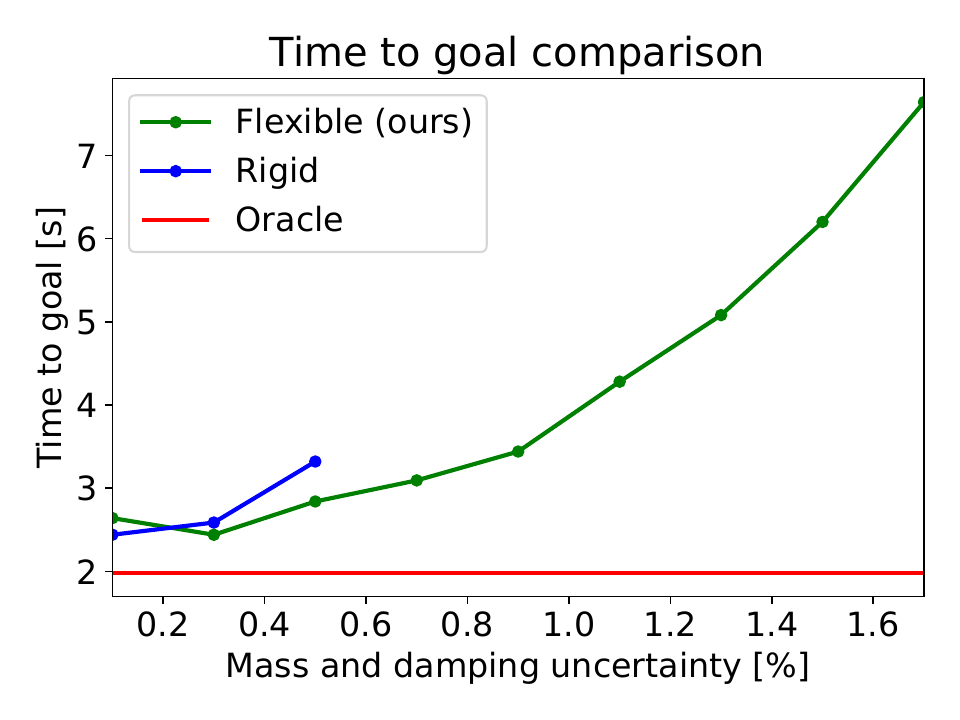}};
\node[inner sep=0pt] at (11, 0) {\includegraphics[width=.33\textwidth]{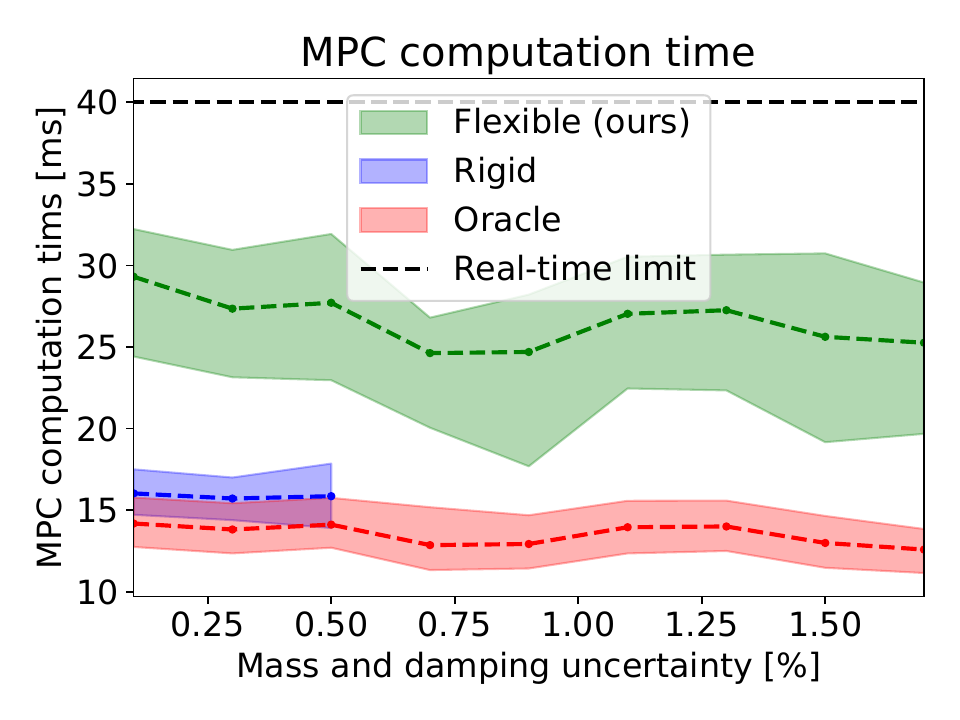}};
\end{tikzpicture}
\caption{
\BW{
\textbf{Left:}
Illustrates our problem scenario, where a 6 DOF robot is supposed to pick an object from a bin and place it on a shelf. The wrist frame position for the resulting motion is illustrated as the blue curve. The starting pick configuration is illustrated as the transparent red mesh. The goal shelving configuration is illustrated as the green transparent mesh. All obstacles are colored gray. 
\textbf{Middle:}
Presents the mean time to goal for the different MPC methods. The horizontal axis represents the amount of uncertainty of mass and damping around nominal values. The vertical axis presents the time to goal. Oracle refers to a nominal MPC without any model errors, Rigid refers to a fixed sized tube, Flexible refers to our approach. A nominal MPC with model error was also tested, but it failed to reach the goal region due to infeasibility. 
\textbf{Right:}
Presents the computation time for the MPC methods in the experiments. The dashed line represents the mean, the shaded region represents the min and max computation times. The black dashed line represents the real time limit of 40 [ms] for the robust methods, that is, our and the rigid approach.
}}
\label{fig:exp}
\end{figure*}
The robot we use for our simulations is an IRB 1100 which is a 6 DOF robot, illustrated in the left part of Figure~\ref{fig:exp}.
\BW{The configuration space $\setC$ is defined by its joint limits and the velocity constraint set is defined as $\setCdot = \{ \qdot \in \setRconf \: | \: \norm{\qdot}_\infty  \le 2 \}$. The dynamics is based on a proprietary model of an IRB 1100 robot with an additional damping term proportional to the velocity, defined as} $\text{diag}([10^{-1}, 10^{-1}, 10^{-1}, 10^{-2}, 10^{-2}, 10^{-4}])~\cdot~2$. \BW{We simulate the continuous-time dynamics~\eqref{eq:cont_dyanamic} with an Explicit Runge-Kutta method of order 5. We consider parametric uncertainty in the mass of each link and the damping. Gravity error is assumed to be negligible, i.e. $\gravDelta\approx0$ in $\eqref{eq:me_c}$, since it is usually easy to compensate with a disturbance observer. We demonstrate our proposed method on a common robotic application, illustrated in the left part of Figure \ref{fig:exp}, which shows a pick and place scenario.
}
\BW{\subsubsection{Convex acceleration set}
We compute the convex acceleration constraint $\setAcc$ in~\eqref{eq:func_fl_cond} using sampling. 
We start with a nominal box constraint set for the acceleration, $\setAcc = \{\a \in \setRconf \: | \: \norm{\a}_\infty \le 20 \}$, represented in vertex form. Then, we uniformly sample $10^5$ states $(\q, \qdot) \in \setX$. For each sampled state and control input vertex $\a\in\setAcc$, we verify condition~\eqref{eq:func_fl_cond}. If the condition is violated, we shrink the acceleration set uniformly by 1~\% and then repeat the process until it is satisfied for all states and control inputs. 
\subsubsection{Model error constants}
The model error constants are also computed using a sampling-based approach. We compute a batch containing $10^6$ random states and control inputs, from which we compute \eqref{eq:me_a}-\eqref{eq:me_c} by using a running max. We check for convergence between the batches by checking the largest difference of the constants in-between the batches. The process is stopped if the difference is less than $10^{-5}$.
}
\subsubsection{Corridor planning}
\BW{
To produce a corridor, we use the hybrid solution presented in \cite{pbrm}. That is, we fist learn a deep neural network representing the SCDF, referred to as an nSCDF, by over-approximating the wrist and tool with a sphere to reduce the dimensionality. When the nSCDF is negative, we fallback on a conventional collision-detector to refine the collision query. To find an initial path, we start by creating a roadmap where the nSCDF is positive. Then, we connect the query points to the graph. To produce a corridor around the path, we discretize the path and query the nSCDF for the distances. For parts where the nSCDF is negative, we use the collision detector, where we sample $10^3$ configuration within a ball of size $0.1$ [rad], shrinking the ball if any collision is detected until all points within the ball are collision-free.
}
\BW{\subsubsection{Methods} 
\label{sec:benchmarks}
We compare the following methods.
\begin{itemize}
    \item \textbf{Flexible:} Our approach, i.e. \eqref{eq:mpc} with \eqref{eq:ball_constraint} and \eqref{eq:real_time_initial_tube}.
    \item \textbf{Rigid:} Standard tube MPC, i.e. same as ours but executed with a constant tube size $\delta$ based on the worst-case model mismatch. \NEW{Thus, similar assumption as in \cite{safe_n_fast_rmpc, smg_mpc}.}
    \item \textbf{Nominal:} The MPC in \eqref{eq:mpc} with \eqref{eq:ball_constraint}, but without robustness features, i.e., $\delta=0$. \NEW{It is similar to \cite{ciao}, but extended to multi-body robots.}
    \item \textbf{Oracle:} The same as nominal, but with a perfect model. This provides a lower bound on the performance. 
\end{itemize}}
All methods are executed with the following parameters: $\matQ=\text{diag}((\textbf{1}_{\DimConf} \cdot 10,  \textbf{1}_{\DimConf} \cdot 0.01))$, $\matQ_e = \text{diag}(\textbf{1}_{\DimState} \cdot 10^4)$ and $\matR= \text{diag}(\textbf{1}_{\DimConf} \cdot 10^{-3})$. A constant control input is applied every $\timeStep=10$ [ms]. The robust MPC methods are optimized every $\NrAuxSteps=4$ steps, resulting in a real-time limit of $40$ [ms]. \NEW{The horizon length was set to the largest value that maintained real-time feasibility in trial runs, resulting in $\NrHorizon=15$.}
\subsubsection{Simulations}
\BW{
All MPC schemes are solved on a laptop with an Intel i5-1155G7 CPU. We repeat each simulation $3$ times, re-sampling the uncertainty parameter $\paramME\in\Theta$ uniformly for each run. To evaluate the scaling capabilities, the methods are tested with different levels of uncertainty. If the method was not able to reach the goal within $100$ seconds, it was stopped and labeled as unsuccessful}. To verify the SCDF, we collision-checked all trajectories returned from the methods with a conventional collision checker using the exact obstacle and robot geometries. The robot was defined to have reached the goal if its state is within an $\varepsilon$-ball of radius $0.01$ around the goal state. We measure performance based on the time it takes to reach the goal region and how the method scales with increasing uncertainty.
\subsection{Results}
\label{sec:results}
The offline computation times are presented in Table~\ref{tab:experiments:offline}.
\begin{table}[H]
    \centering
    \caption{Average computation times for the offline tasks.}
    \label{tab:experiments:offline}
    \begin{tabular}{l|l }
    Offline task & Computation time [s] \\
    \hline
    Learn nSCDF & 32000  \\
    Convex acceleration set \eqref{eq:func_fl_cond} & 4  \\
    Model error constants \eqref{eq:me_a}-\eqref{eq:me_c} & 7900  \\
    Computing controller (\ifbool{arxiv}{Appendix \ref{sec:cvx_control_syn}}{\cite[App. B]{arxiv_version_of_paper}}) & \NEW{60}
    \end{tabular}
\end{table}
\noindent\BW{\textit{Corridor}: Finding a path takes $40$ [ms] and to produce a corridor with the hybrid SCDF takes roughly $5-6$ [s]. \textit{Robustness \& performance:} The times to reach the goal for different level of uncertainty is presented in the middle of Figure~\ref{fig:exp}.} None of the methods resulted in trajectories that were in collision. The nominal MPC is not included in the figure since it resulted in infeasibility issues in each run due to the model errors. This shows the importance of including robustness in the design. \BW{For the robust methods, we see an intuitive trend, larger uncertainty yields longer time to reach the goal. This is a design feature to ensure feasible and safe motion. \NEW{For the proposed robust method, we see an intuitive behavior, larger uncertainties, lead to larger constants~\eqref{eq:me_a}-\eqref{eq:me_c}, which lead to more cautious operation, and thus a longer execution time.} Comparing the performance between the robust methods, we see that for lower uncertainties, the methods perform basically the same, but with increasing uncertainties, the tightening becomes too conservative for the rigid rube. Our method shows superior scaling capabilities compared to the rigid tube, being able to scale to more than $3$ times higher level of uncertainty.
\\
\textit{Online complexity:} Assigning the balls and computing the virtual goal was observed to take max 1 [ms]. The main computational bottleneck is in solving the MPC problem. We present the statistics of the solvers computation times in the right part of Figure \ref{fig:exp}. Observing the plot, we see that solving the nominal MPC is done very fast, with a max time of roughly $15$ [ms]. Naturally, adding additional robustness features adds more computation time, roughly a factor $2$. We observe that our flexible tube has a max computation time below $33$ [ms], which therefore fulfills our real-time requirement, since we take $\NrAuxSteps=4$ time steps with the auxiliary controller.
}
\section{Conclusions}
\label{sec:conclusions}
We have presented a novel convex robust motion planning solution for manipulators that gives robustness guarantees to bounded model errors and results in collision-free motion. One of the main benefits is that we derived a convex optimization problem, which can be solved fast and reliably. From the numerical experiments, we observed that a robust design of the MPC is necessary to maintain feasibility. Compared to a more standard robust MPC formulation, our approach was less conservative and scaled to over three times larger levels of uncertainty. 
\bibliographystyle{ieeetr}
\bibliography{main}

@book{boyd_lmi,
  title={Linear matrix inequalities in system and control theory},
  author={Boyd, Stephen and El Ghaoui, Laurent and Feron, Eric and Balakrishnan, Venkataramanan},
  year={1994},
  publisher={SIAM}
}

@inproceedings{krupa2024model,
  title={Model predictive control for tracking using artificial references: Fundamentals, recent results and practical implementation},
  author={Krupa, Pablo and K{\"o}hler, Johannes and Ferramosca, Antonio and Alvarado, Ignacio and Zeilinger, Melanie N and Alamo, Teodoro and Limon, Daniel},
  year={2024},
  booktitle={Proc. Conf. Decision and Control (CDC)}, 
}

@article{kohler2024analysis,
  title={Analysis and design of model predictive control frameworks for dynamic operation—An overview},
  author={K{\"o}hler, Johannes and M{\"u}ller, Matthias A and Allg{\"o}wer, Frank},
  journal={Annual Reviews in Control},
  volume={57},
  pages={100929},
  year={2024},
  publisher={Elsevier}
}

@article{sasfi2023robust,
  title={Robust adaptive {MPC} using control contraction metrics},
  author={Sasfi, Andr{\'a}s and Zeilinger, Melanie N and K{\"o}hler, Johannes},
  journal={Automatica},
  volume={155},
  pages={111169},
  year={2023},
  publisher={Elsevier}
}

@book{rob_mod_plan_cont,
author = {Siciliano, Bruno and Sciavicco, Lorenzo and Villani, Luigi and Oriolo, Giuseppe},
title = {Robotics: Modelling, Planning and Control},
year = {2008},
publisher = {Springer Publishing Company}
}

@inproceedings{ciao,
  title={An \text{NMPC} approach using convex inner approximations for online motion planning with guaranteed collision avoidance},
  author={Schoels, Tobias and Palmieri, Luigi and Arras, Kai O and Diehl, Moritz},
  booktitle={IEEE International Conference on Robotics and Automation (ICRA)},
  pages={},
  year={2020},
  organization={}
}

@book{mpc_book,
  title={Predictive control for linear and hybrid systems},
  author={Borrelli, Francesco and Bemporad, Alberto and Morari, Manfred},
  year={2017},
  publisher={Cambridge University Press}
}

@misc{pbrm,
      title={A neural signed configuration distance function for path planning of picking manipulators}, 
      author={Bernhard Wullt and Mikael Norrlöf and Per Mattsson and Thomas B. Schön},
      year={2025},
      eprint={2502.16205},
      archivePrefix={arXiv},
      primaryClass={cs.RO},
      url={"\url{https://arxiv.org/abs/2502.16205}"}, 
     howpublished = "\url{https://arxiv.org/abs/2502.16205}"
}

@article{set_point_tracking_mpc,
  title={Set-point tracking MPC with avoidance features},
  author={dos Santos, Marcelo Alves and Ferramosca, Antonio and Raffo, Guilherme Vianna},
  journal={Automatica},
  volume={159},
  pages={111390},
  year={2024},
  publisher={Elsevier}
}

@misc{arxiv_version_of_paper,
      title={Robust Convex Model Predictive Control with collision avoidance guarantees for robot manipulators}, 
      author={Bernhard Wullt and Johannes K{\"o}hler and Per Mattsson and Mikeal Norrl{\"o}f and Thomas B. Sch{\"o}n},
      year={2025},
      eprint={2508.21677},
      archivePrefix={arXiv},
      primaryClass={cs.RO},
      howpublished="\url{https://arxiv.org/abs/2508.21677}", 
}

@book{lavallebook,
  title={Planning algorithms},
  author={LaValle, Steven M},
  year={2006},
  publisher={Cambridge university press}
}

@article{rrt,
  title={Rapidly-exploring random trees: A new tool for path planning},
  author={LaValle, Steven},
  journal={Research Report 9811},
  year={1998},
  publisher={Department of Computer Science, Iowa State University}
}

@ARTICLE{topps_cvx,
  author={Verscheure, Diederik and Demeulenaere, Bram and Swevers, Jan and De Schutter, Joris and Diehl, Moritz},
  journal={IEEE Transactions on Automatic Control}, 
  title={Time-Optimal Path Tracking for Robots: A Convex Optimization Approach}, 
  year={2009},
  volume={54},
  number={10},
  pages={2318-2327}
}

@article{cspf,
  title={Configuration Space Distance Fields for Manipulation Planning},
  author={Li, Yiming and Chi, Xuemin and Razmjoo, Amirreza and Calinon, Sylvain},
  journal={arXiv preprint arXiv:2406.01137},
  year={2024}
}

@article{IRIS_nlp,
  title={Growing convex collision-free regions in configuration space using nonlinear programming},
  author={Petersen, Mark and Tedrake, Russ},
  journal={arXiv preprint arXiv:2303.14737},
  year={2023}
}

@inproceedings{iris_c,
  title={Finding and optimizing certified, collision-free regions in configuration space for robot manipulators},
  author={Amice, Alexandre and Dai, Hongkai and Werner, Peter and Zhang, Annan and Tedrake, Russ},
  booktitle={International Workshop on the Algorithmic Foundations of Robotics},
  pages={328--348},
  year={2022},
  organization={Springer}
}

@book{horn_john_matrix,
  title={Matrix analysis},
  author={Horn, Roger A and Johnson, Charles R},
  year={2012},
  publisher={Cambridge university press}
}

@ARTICLE{kdf,
  author={Verginis, Christos K. and Dimarogonas, Dimos V. and Kavraki, Lydia E.},
  journal={IEEE Transactions on Robotics}, 
  title={\text{KDF}: Kinodynamic Motion Planning via Geometric Sampling-Based Algorithms and Funnel Control}, 
  year={2023},
  volume={39},
  number={2},
  pages={978-997}
}

@ARTICLE{safe_n_fast_rmpc,
  author={Nubert, Julian and K{\"o}hler, Johannes and Berenz, Vincent and Allg{\"o}wer, Frank and Trimpe, Sebastian},
  journal={IEEE Robotics and Automation Letters}, 
  title={Safe and Fast Tracking on a Robot Manipulator: Robust \text{MPC} and Neural Network Control}, 
  year={2020},
  volume={5},
  number={2},
  pages={3050-3057},
  doi={10.1109/LRA.2020.2975727}
}

@ARTICLE{ddmpc,
  author={Carron, Andrea and Arcari, Elena and Wermelinger, Martin and Hewing, Lukas and Hutter, Marco and Zeilinger, Melanie N.},
  journal={\text{IEEE Robotics and Automation Letters}},
  title={Data-Driven Model Predictive Control for Trajectory Tracking With a Robotic Arm}, 
  year={2019},
  volume={4},
  number={4},
  pages={3758-3765},
  doi={10.1109/LRA.2019.2929987}
}

@ARTICLE{smg_mpc,
  author={Incremona, Gian Paolo and Ferrara, Antonella and Magni, Lalo},
  journal={IEEE/ASME Transactions on Mechatronics}, 
  title={\text{MPC} for Robot Manipulators With Integral Sliding Modes Generation}, 
  year={2017},
  volume={22},
  number={3},
  pages={1299-1307}
}

@article{tubeMPC,
title = {Robust model predictive control of constrained linear systems with bounded disturbances},
journal = {Automatica},
volume = {41},
number = {2},
pages = {219-224},
year = {2005},
issn = {0005-1098},
doi = {https://doi.org/10.1016/j.automatica.2004.08.019},
url = {https://www.sciencedirect.com/science/article/pii/S0005109804002870},
author = {D.Q. Mayne and M.M. Seron and S.V. Raković}
}

@article{findeisen2004computational,
  title={Computational delay in nonlinear model predictive control},
  author={Findeisen, Rolf and Allg{\"o}wer, Frank},
  journal={IFAC Proceedings Volumes},
  volume={37},
  number={1},
  pages={427--432},
  year={2004},
  publisher={Elsevier}
}

@article{mpGcs,
  title={Motion planning around obstacles with convex optimization},
  author={Marcucci, Tobia and Petersen, Mark and von Wrangel, David and Tedrake, Russ},
  journal={Science robotics},
  volume={8},
  number={84},
  pages={},
  year={2023},
  publisher={American Association for the Advancement of Science}
}

@inproceedings{FER,
  title={\text{LSTM}-based inverse dynamics learning for franka emika robot},
  author={Schneider, Jan-Niklas and Gori{\ss}en, Leon and Kaster, Thomas and Walderich, Philipp and Hinke, Christian},
  booktitle={2024 International Conference on Control, Automation and Diagnosis (ICCAD)},
  pages={1--6},
  year={2024},
  organization={IEEE}
}

@inproceedings{llm3,
  title={\text{LLM3}: Large language model-based task and motion planning with motion failure reasoning},
  author={Wang, Shu and Han, Muzhi and Jiao, Ziyuan and Zhang, Zeyu and Wu, Ying Nian and Zhu, Song-Chun and Liu, Hangxin},
  booktitle={2024 IEEE/RSJ international conference on intelligent robots and systems (IROS)},
  pages={12086--12092},
  year={2024},
  organization={IEEE}
}

\ifappendix
\appendix
\subsection{Model error derivation}
\label{sec:model_err_der}
In the following, to increase readability, we drop the input arguments to the functions, e.g. $\matM = \matM(\q)$. From the manipulator dynamics \eqref{eq:dyn} we obtain
\begin{align*}
    \qddot &= \matM^{-1}(\u - \matC \qdot - \grav) \\
           &\stackrel{\text{FL} \; \eqref{eq:fdb_lin_map}}{=} \matM^{-1}(\matMnom \a + \matCnom \qdot + \gravNom - (\matCnom + \matCerr ) \qdot - (\gravNom + \gravErr))  \\
           &= \matM^{-1}\matMnom \acc - \matM^{-1}\matCerr \qdot - \matM^{-1} \gravErr \\
           & \stackrel{\star}{=} (\matI + \matMdelta) \acc - \matM^{-1}\matCerr \qdot - \matM^{-1}\gravErr \\
           &= \acc + \matMdelta\acc - \matM^{-1}\matCerr \qdot - \matM^{-1}\gravErr \\
           &= \a + \funcModelError(\q, \qdot, \acc).
\end{align*}
In the above, FL denotes feedback linearization. The $\star$ denotes the step where we decompose the matrix in front of the acceleration input into $\matI + \matMdelta$ as follows
\begin{align}
    &(\matMnom + \matMerr)^{-1} \matMnom = \matI + \matMdelta \nonumber \\
    &\matMnom = (\matMnom + \matMerr) +  (\matMnom + \matMerr)\matMdelta \nonumber \\
    &\matMdelta = -(\matMnom + \matMerr)^{-1} \matMerr
\end{align}
The function $\funcModelError(\q, \qdot, \acc) : \setR^{\DimState} \times \setR^{\DimConf} \mapsto \setRconf$ defines the state and input dependent model error on the following form
\begin{align}
    \Delta(\q, \qdot, \acc) &= \matMdelta \acc + \matCdelta \qdot + \gravDelta, \\
    \matMdelta &= - \matM^{-1} \matMerr, \\ 
    \matCdelta &= - \matM^{-1} \matCerr, \\
    \gravDelta &= - \matM^{-1} \gravErr.
\end{align}
\subsection{Auxiliary controller and Lyapunov matrix}
\label{sec:cvx_control_syn}
We present the optimization problem to compute $\matP,\matK$ in Appendix~\ref{sec:cvx_control_syn_opt_prob}, and derive it in Appendix~\ref{sec:derive_opt_prob}. How we selected a suitable pair $\matP$ and $\matK$ in the experiments is presented in Appendix~\ref{sec:opt_candidate_selection}.
\subsubsection{Convex optimization program}
\label{sec:cvx_control_syn_opt_prob}
We solve the following semidefinite program to produce the gain $\matK$ and Lyapunov matrix $\matP$:
\begin{subequations}
\label{eq:lmi}
\begin{align}
\underset{\matE, \matY, c^2_{x, i}, c^2_{u, j}, \bar{w}^2 }{\operatorname{min}} 
    \quad 
    & 
    L(\bar{w}^2, c^2_{x, 1}, \hdots, c^2_{x, m}, c_{u, 1}^2, \hdots, c_{u, n}^2)
    &
    \\
\textrm{s.t.} 
\quad 
& 
\begin{bmatrix}
\rho^2 \matE, \: (\matA\matE + \matB\matY)^\top \\
(\matA\matE + \matB\matY)^\top, \: \matE
\end{bmatrix}
\succcurlyeq
0 
& 
\\
& 
\begin{bmatrix}
c_{x, i}^2, \: [\matA_x]_i \matE \\
([\matA_x]_i \matE)^\top, \: \matE
\end{bmatrix}
\succcurlyeq
0, \forall i \in \setN_{1:m}, \label{eq:lmi:lmi_state} \\
& 
\begin{bmatrix}
c_{u, i}^2, \: [\matA_u]_i \matY \\
([\matA_u]_i \matY)^\top, \: \matE
\end{bmatrix}
\succcurlyeq
0, \forall j \in \setN_{1:n},  \label{eq:lmi:lmi_control} \\ 
& 
\begin{bmatrix}
\bar{w}^2, \: \v_{\setW}^\top  \label{eq:lmi:lmi_disturbance} \\
\v_{\setW}, \: \matE
\end{bmatrix}
\succcurlyeq
0, \forall \v_{\setW} \in \vertexrep{\setW},
\end{align}
\end{subequations}
with the objective function defined as 
\begin{align}
    \label{eq:lmi:loss}
    & L(\bar{w}^2, c^2_{x, 1}, \hdots, c^2_{x, m}, c_{u, 1}^2, \hdots, c_{u, n}^2) = \nonumber \\
    & \frac{1}{2 ( 1 - \rho)}
    \left (
    (m + n) \bar{w}^2
    +
    \sum_{i=1}^{m} 
    c_{x, i}^2
    + 
    \sum_{i=1}^{n} 
    c_{u, i}^2
    \right ).
\end{align}
The gain and Lyapunov matrix is computed from the following relationship $\matE = \matP ^{-1}$ and $\matY = \matK \matE$. The decision variables $c_{x, i}^2, i \in \setN_{1:m}$ and $c_{u, j}^2, j \in \setN_{1:n}$, control the amount of tightening on the state and control constraints, respectively. In this context, the set $\setW$ is a box set of the model error, including both the model error due to model uncertainty \eqref{eq:model_error_exp}, and the discretization error \eqref{eq:dynamics_discretized}. Here $\vertexrep{\cdot}$ corresponds to the vertices of this box. The inputs are the state constraints, $\matAx \in \setR^{m \times \DimState}$ and $\b_x \in \setR^{m}$, the control input constraints, $\matAu \in \setR^{n \times \DimConf}$ and $\b_u \in \setR^{n}$, and the contraction rate, $\rho\in(0,1)$. We assume polytopic constraints, represented in its half-plane form, i.e. for the state constraints the form is
\begin{equation}
\setX = \{\x \in \setR^{\DimState} \: | \: [\matAx]_i \x \le b_{x, i}, \: i \in \setN_{1:m} \},
\end{equation}
where $[\matAx]_i$ is the $i$:th row of the matrix $\matAx$ and $b_{x, i}$ is the $i$:th element of $\bx$.
\subsubsection{Derivation of optimization problem}
\label{sec:derive_opt_prob}
Our goal is to derive an optimization problem where the amount of tightening is included in the cost, thereby allowing us to reduce the conservatism in the tightening. 
\\\\
In order to achieve this, we first have to derive an expression for the tightened state and input constraints. Then, we present LMI's which allows us to include the tightening in the optimization. This results in a non-convex objective which we as a last step convexify, ending up with our proposed optimization problem.
\\\\
First, the following linear matrix inequalities (LMIs) are a standard reformulation of inequality~\eqref{eq:contraction_req} (cf.~\cite{boyd_lmi}):
\begin{subequations}
\label{eq:lmi_vaniall}
\begin{align}
\begin{bmatrix}
\rho^2 \matE, \: (\matA\matE + \matB\matY)^\top \\
(\matA\matE + \matB\matY)^\top, \: \matE
\end{bmatrix}
\succcurlyeq
0. 
\end{align}
\end{subequations}
The worst case disturbance is defined as
\begin{equation}
    \wWC = \max_{\w \in \setW} \norm{\w}_{\matP},
\end{equation}
A robust positive invariant (RPI) set can be computed that contains the worst case disturbance, having the following tube size
\begin{equation}
    \label{eq:wc_tube}
    \bar{\tubeSize} = \frac{\bar{w}}{1-\rho}.
\end{equation}
The rigid tube MPC in Section~\ref{sec:benchmarks} uses the above tube size to tighten its constraints.
\\\\
Next, we focus on how to introduce the tightenings into the optimization problem. We start by deriving an expression for the state constraint tightening. We require that around a reference state, $\xNom$, the tube with size $\tubeSize$ should not violate the state constraints. That is, for each $i \in \setN_{1:m}$, we want
\begin{equation}
    \label{eq:state_const_cond}
    [\matAx]_i \x \le b_{x, i}, \forall \x \: : \:  \norm{\x - \xNom}_\matP \le \tubeSize.
\end{equation}
The error is defined as $\e = \x - \xNom$. We introduce
\begin{equation}
    \e = \matPinvSqrt \eTrf.
\end{equation}
Now, we input the above into the condition of \eqref{eq:state_const_cond} resulting in
\begin{equation}
    [\matAx]_i(\xNom + \matPinvSqrt \eTrf) \le b_{x, i}, \: \forall \norm{\eTrf} \le \tubeSize.
\end{equation}
The worst case error maximizer is the following
\begin{equation}
    \eTrfStar = \frac{( [\matAx]_i \matPinvSqrt)^\top}{\norm{([\matAx]_i\matPinvSqrt)^\top}} \: \tubeSize.
\end{equation}
Thus, the tightened constraints become
\begin{equation}
    [\matAx]_i \xNom + \| ([\matAx]_i \matPinvSqrt)^\top \|  \: \tubeSize \le b_{x, i},
\end{equation}
and we define the tightening constants as
\begin{equation}
    \label{eq:state_tightening}
    c_{x, i} = \norm{([\matAx]_i \matPinvSqrt)^\top}.
\end{equation}
Next, we focus on the control input tightening. Using the control law in \eqref{eq:aux_control_law}, $\acc =\accNom + \matK(\x-\xNom)$, into the above we get
\begin{equation}
    [\matAu]_j \accNom + [\matAu]_j \matK \e \le b_{u,j}.
\end{equation}
This is on the same form as for the state constraints. Thus, following the same approach, the tightening constant can be expressed as
\begin{equation}
    \label{eq:cont_tightening}
    c_{u, j} = \norm{([\matAu]_j \matK \matPinvSqrt)^\top}.
\end{equation}
Having introduced the expressions for both the state and control input tightening, we now address how to include them into the optimization problem. We start by re-writing our tightenings. Focusing on the state inputs, for $i \in \setN_{1:m}$ then \eqref{eq:state_tightening} and $\eqref{eq:wc_tube}$ defines our tightening, which we express as
\begin{align}
    c_{x, i} \bar{\tubeSize} = c_{x,i} \bar{w} \frac{1}{1-\rho},
\end{align}
and split into two inequalities
\begin{align}
    [\matA_{x}]_i \matE [\matA_{x}]_i^\top &\le c_{x,i} ^2, \\
    \max_{\w \in \setW } \sqrt{\w^\top \matP \w } &\le \wWC,
\end{align}
where $\matE=\matP^{-1}$. We rewrite the above into LMI's using the Schur complement \cite{horn_john_matrix}, ending up with the LMI's in \eqref{eq:lmi:lmi_state} and \eqref{eq:lmi:lmi_disturbance}. The control input tightenings follows the same reasoning. Now, we include all the state and control input tightenings in the objective, resulting in
\begin{equation}
    \frac{1}{1-\rho}(\sum^m_{i=1} c_{x,i} \bar{w} +\sum^n_{i=j} c_{u,j} \bar{w}).
\end{equation}
The bi-linear terms makes the cost non-convex. To make it convex we use the inequality of the arithmetic and geometric means \cite{horn_john_matrix}, which results in
\begin{align}
    c_{x,i} \bar{w} &\le \frac{1}{2} ( c_{x,i}^2 + \bar{w}^2), \\
    c_{u,j} \bar{w} &\le \frac{1}{2} ( c_{u,j}^2 + \bar{w}^2),
\end{align}
ending up with the loss in \eqref{eq:lmi:loss}.
\\\\
\NEW{Since the LTI model and constraints have a block structure, we can decompose this LMI into smaller separate LMIs, which are solved more efficiently. The details can be seen in the accompanying open-source code.}
\subsubsection{Candidate selection}
\label{sec:opt_candidate_selection}
We normalized the objective function by dividing the constraint tightenings of the configuration, velocity and control with a representative normalizing factor. For the configuration tightening, we used $0.1$ [rad], which was the padded clearance added in the path planning. For the velocity and control tightening, they were divided with their corresponding max values from the constraints, i.e. $2$ [rad/s] and $20$ [rad/s$^2$], respectively.
\\\\
We compute $20$ values of $\rho$, equally spaced between 0.8 and 0.99. Having solved the optimization problem for all values of $\rho$ resulted in $20$ pairs of $\matK$ and $\matP$. First, all pairs where $\tubeGrowth \ge 1$ was satisfied were removed. From the remaining pairs, we selected the pair that resulted in the smallest max tightening. For the rigid tube MPC, the condition $\tubeGrowth \ge 1$ was ignored, otherwise the same selection rule was used.
\subsection{Proof of Proposition 2}
\label{sec:proof_cl_error_prop}
It holds that
\begin{align*}
\norm{\x_+-\xNom_+}_\matP &= \norm{\matAcl (\x - \xNom) + \matB \Delta(\x, \acc)+ \funcDiscError(\x, \acc)}_\matP \\
&\leq \norm{\matAcl (\x - \xNom)}_\matP +\norm{\matB \Delta(\x, \acc) + \funcDiscError(\x, \acc)}_\matP \\
& \hspace{-17.5pt} \stackrel{\eqref{eq:contraction_req},~\mathrm{Prop.}~\ref{prop:delta_beta}}{\leq} \rho \norm{\x - \xNom}_\matP+\beta(\x, \acc),
\end{align*}
with $\matAcl=\matA+\matB\matK$. 
The uncertainty bound $\beta$ satisfies
\begin{align*}
\beta(\x,\acc)-\beta(\xNom, \accNom) &\stackrel{\eqref{eq:model_error_upper_bound}}{=}a(\norm{\acc}-\norm{\accNom})+b(\norm{ \matV\x}-\norm{\matV \xNom}) \\
&\leq  a\norm{\matK(\x-\xNom)}+b\norm{\matV(\x-\xNom)} \\
&\stackrel{\eqref{eq:L_beta}}{\leq} L_\beta \norm{\x-\xNom}_{\matP} \leq L_\beta \delta.
\end{align*} 
Combining both bounds yields
\begin{align*}
\norm{\x_+-\xNom_+}_\matP & \leq \rho \norm{\x - \xNom}_\matP+\beta(\x, \acc)\\
&\leq \rho \delta + L_\beta \delta + \beta(\xNom,\accNom)=\delta_+. \qedhere
\end{align*}
\subsection{Tube in ball constraint}
\label{sec:tube_in_ball_const}
To compute the tightened ball constraints in~\eqref{eq:ball_constraint}, we compute a ball that over-approximates the projection of the tube on the configuration space:
\begin{align}
    \label{eq:tube_in_ball_cond}
    \setBQ(\tubeSize) \subseteq \{ \q(\x) \in \setR^{\DimConf} \: | \norm{\q} \le \tubeSize\cdot r_p \}
\end{align}
with a suitable radius $r_p>0$. 
To ensure the above, we start by projecting $\setE$ onto the configuration space. The Lyapunov matrix is structured as
\begin{equation}
    \matP
    = 
    \begin{bmatrix}
    \matP_{11},  \matP_{12} \\
    \matP_{21},  \matP_{22} \\
    \end{bmatrix}.
\end{equation}
The projection of the ellipsoid onto the configuration space, $\setE_{\q}$, is done with the Schur-complement
\begin{equation}
    \matP_{\q} = \matP_{11} - \matP_{12} \matP_{22}^{-1} \matP_{21}.
\end{equation}
The eigenvalues give the principal axes of the resulting ellipsoid. We compute a radius that encompasses the projected ellipsoid as
\begin{equation}
    \label{eq:radius_ball_const}
    r_p = 1 / \sqrt{\lambda_{\text{min}}(\matP_{\q})}.
\end{equation}
Now, to fulfill condition \eqref{eq:tube_in_ball_cond}, we simply shrink the given ball's radius, $r$, by $r_p \tubeSize$. Thus our homothetic constraint tightening becomes
\begin{equation}
    \setB \ominus \setBQ(\tubeSize) = \{\q \in \setRconf \: | \: \norm{\bubbleCenter-\q} \le r - r_p \tubeSize \}.
\end{equation}
\subsection{Proof of Theorem~\ref{thm:robust_MPC_coll_avoid}}
\label{app:proof_mpc_complete}
\BW{The proof follows the same steps as Theorem~\ref{thm:robust_MPC} and we only highlight the differences related to the added and updated ball constraints~\eqref{eq:ball_constraint} and the virtual goal~\eqref{eq:assign_vgoal}.\\
\textbf{Part I:}
The system starts at steady-state, the allocation rule~\eqref{eq:assign_ball_rule} picks balls with largest margin around the initial trajectory. 
Note that the corridor satisfies~\eqref{eq:condition_corridor}, which also ensures that the ball $\setB$ around the initial condition is larger than $\setE_{\q}(2\epsilon+\delta_f)$. 
Hence, a feasible solution to Problem~\eqref{eq:mpc} with the added obstacles avoidance constraints~\eqref{eq:ball_constraint} is given by staying at the steady-state with $\acc=\zeroVec_{\DimConf}$, i.e., we are guaranteed that the MPC problem is feasible at the start.\\
Next, we show recursive feasibility. 
Analogous to the proof of Theorem~\ref{thm:robust_MPC}, the candidate solution is given by shifting the previous optimal solution by $\NrAuxSteps$ steps.
Compared to Theorem~\ref{thm:robust_MPC}, we need to account for the change in the initial condition~\eqref{eq:predict_tube_size} and the added obstacle avoidance constraints~\eqref{eq:ball_constraint}. 
For the initial constraint~\eqref{eq:predict_tube_size}, the shifted candidate solution $\delta_0=\delta^\star_{\NrAuxSteps}$, $\xNom_0=\xNom^\star_{\NrAuxSteps}$ also provides a feasible solution.  
In particular, Proposition~\ref{prop:tube_dynamics} also ensures that $\hat{\tubeSize}_0\leq \tubeSize^\star_0$ and monotonicity of~\eqref{eq:htmpc:tube_dyn} ensures $\hat{\tubeSize}_{\NrAuxSteps}\leq \tubeSize^\star_{\NrAuxSteps}$.
\\
Regarding the added ball constraints~\eqref{eq:ball_constraint}: If we would simply shift the previously allocated balls, i.e., $\setB_i\leftarrow\setB_{i+\NrAuxSteps}$, then the fact that the candidate sequence is equally shifted compared to the previous feasible solution, ensures feasibility of the candidate sequence, 
The ball assignment~\eqref{eq:assign_ball_rule} is such that the distance to the boundary of the ball is non-decreasing and thus~\eqref{eq:ball_constraint} is also feasible with the new assigned balls. \\
\textbf{Part II:} For all $i\in \setN_{0:\NrAuxSteps-1}$, $\norm{\x(k+i)-\xNom^\star_i}_{\matP} \leq \delta_i^\star$ due to the initial constraint~\eqref{eq:real_time_initial_tube} and the tube construction~\eqref{eq:htmpc:tube_dyn}.
Hence, the tightened constraints~\eqref{eq:htmpc:limits_states}, \eqref{eq:htmpc:limits_controls}, and~\eqref{eq:ball_constraint} ensure $\q(k)\in\setB$, $\x(k)\in\setX$, $\a(k)\in\setAcc$.
Condition~\eqref{eq:func_fl_cond} ensures $\u(k)\in\mathcal{U}$, i.e., the torque limits are respected at each discrete time.
The construction of the balls~\eqref{eq:cball} with the SCDF~\eqref{eq:SCDF} ensures $\q(k)\in\setCf$, i.e., the robot operates in a collision-free configuration,\\
\textbf{Part III:}
We need to ensure that $\xGoalVirtual$ converges to $\xGoal$ in finite time. 
For contradiction, suppose that the intermediate goal $\xGoalVirtual=\c_i$ is not updated for some arbitrarily long time. 
Then, $\xGoalVirtual-\xNom_0^\star$ converges exponentially to zero, analogous to the proof of Theorem~\ref{thm:robust_MPC}, given that $\qGoalVirtual\in(\setC \cap \setB_{\NrHorizon}) \ominus \setE_\q(\epsilon+\tubeSize_f)$.
Thus, in finite time, $\norm{\xNom_{\NrHorizon}^\star-\xGoalVirtual}_\matP \leq \epsilon$, for any $\epsilon>0$. 
From~\eqref{eq:condition_corridor}, we know that $\c_{i+1}\in\setB(\c_i)\ominus \setE_\q(2\epsilon + \tubeSize_f)$.
The ball selection~\eqref{eq:assign_ball_rule} for $\setB_{\NrHorizon}$ maximizes the distance around $\xNom_{\NrHorizon}$, which in combination with $\norm{\c_i-\xNom_{\NrHorizon}}_\matP\leq \epsilon$ ensures that $\c_{i+1}\in\setB_{\NrHorizon}\ominus\setE_\q(\epsilon + \tubeSize_f)$. 
Hence, $\c_{i+1}$ satisfies condition~\eqref{eq:assign_vgoal}, ensuring that the virtual goal is updated. Thus, $\xGoalVirtual$ is updated in finite time. 
Given the finite number of possible virtual goals based on the discretization,  $\xGoalVirtual=\xGoal$ in finite time. 
Convergence of $\x(k)$ to the neighborhood of the goal follows analogous to Theorem~\ref{thm:robust_MPC}.}
%
%
%
%
\NEW{
\subsection{Numerical experiments with data from hardware experiments}
\label{sec:dd}
The following section aims to show that our uncertainty relationship \eqref{eq:model_error_upper_bound} can be seen in real-world data and that our planner works under realistic uncertainty levels. We show this by first presenting a data driven approach to compute the model error constants in \eqref{eq:me_a}-\eqref{eq:me_c} directly from data obtained from hardware experiments. We then verify that our approach still outperforms the used benchmarks from Section \ref{sec:benchmarks}.
\\\\
We describe our setup in Section~\ref{sec:dd:setup}. Next, we present our data source and how to extract the relevant data from it in Section~\ref{sec:dd:dataset}. Then, in Section \ref{sec:dd:mec}, we move over to the main methodology of how to estimate the model-error constants from the data. Section~\ref{sec:dd:sim} presents the setup for the simulations, ending with presenting the results in Section~\ref{sec:dd:res}.
\subsubsection{Setup}
\label{sec:dd:setup}
For the experiments, we consider a Franka Emika robot which is a 7 DOF robot. We define the configuration space, velocity and acceleration constraints as $\setC = \{ \qdot \in \setRconf \: | \: \norm{\qdot}_\infty  \le  \pi \} $, $\setCdot = \{ \qdot \in \setRconf \: | \: \norm{\qdot}_\infty  \le 2 \}$ and $\setAcc = \{\a \in \setRconf \: | \: \norm{\a}_\infty \le 20 \}$, respectively.
\subsubsection{Data processing}
\def\setD{\mathcal{D}}
\label{sec:dd:dataset}
We use the publicly available data set from \cite{FER}, which contains closed-loop state trajectories measured with a sampling time of $10$ [ms]. The desired acceleration is sampled with a different sampling time, hence, we linearly interpolate it to that of the measured data, yielding a coherent dataset. We extract the state, the desired acceleration, and the state in the subsequent time step. From the relationship in \eqref{eq:dynamics_discretized}, the deviation in the predicted and observed subsequent state defines our empirical model mismatch. Since our focus is on model errors, we neglect mismatch in the position, which is primarily due to measurement noise. Repeating the calculations for all the data points results in the data set
\begin{equation}
     \setD = \{ (\x_i, \a_i, \Delta_i) \}^N_{i=1},
\end{equation}
where $N$ is the number of data points.
\subsubsection{Model error constants}
\label{sec:dd:mec}
\begin{figure}
    \centering
    \includegraphics[width=1.0\linewidth, trim={0.5cm 5cm 0.5cm 5.0cm},clip]{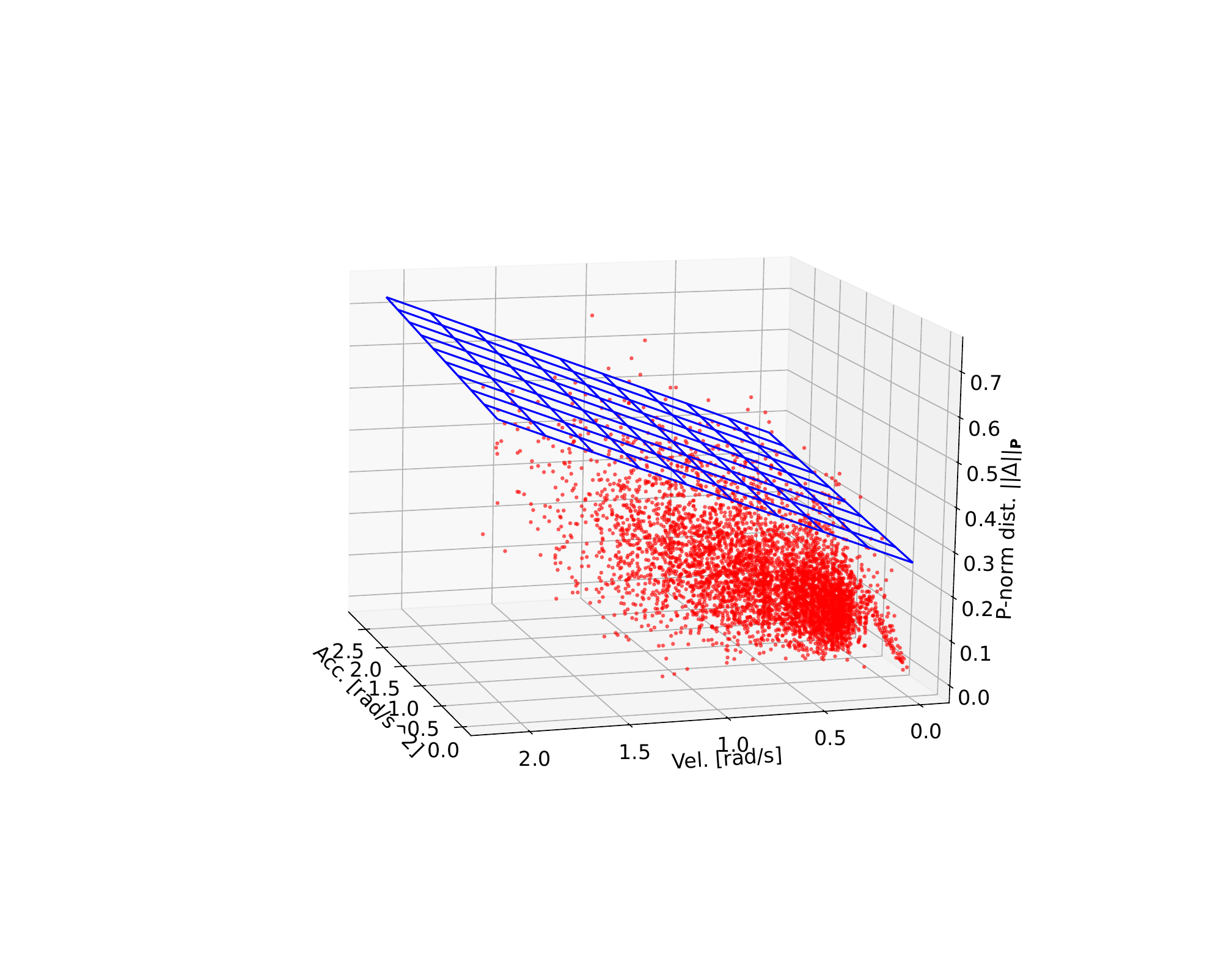}
    \caption{Shows our resulting conservative estimator (blue plane), where the constants are computed from measurement data. The visualization shows a random subset of the used measurement data (red points).}
    \label{fig:bounding_plane}
\end{figure}
Having compiled the dataset, the next step is to find a suitable auxiliary controller with a corresponding bound on the disturbances. 
\\\\
Assuming that $\matP$ and $\matK$ is given, we compute the model error constants, $a$, $b$ and $c$ in \eqref{eq:me_a}-\eqref{eq:me_c}, by reformulating \eqref{eq:model_error_upper_bound} as a linear regression problem, i.e., we try approximate the disturbance according to
\begin{equation}
    ||\Delta||_\matP \approx \tilde{\beta}(\x, \a) := a \norm{\acc} + b \norm{\qdot(\x)} + \tilde{c}.
\end{equation}
The constants $a$, $b$ and $\tilde{c}$ are computed by solving the least square problem 
\begin{equation}
    (a, b, \tilde{c}) = \underset{a, b, \tilde{c}}{\text{argmin}} \frac{1}{N} \sum_{(\x_i, \acc_i, \Delta_i) \in \setD} \norm{\tilde{\beta}(\x_i, \acc_i) - \norm{\Delta_i}_\matP}^2.
\end{equation}
Afterwards, the bias term, $\tilde{c}$, is increased to make $\beta$ a conservative over-approximation of the model error $\Delta$. Specifically, we set 
\begin{equation}
    c = \tilde{c} + 2 \cdot \sigma,
\end{equation}
where $\sigma$ is the standard deviation of our prediction error. This yields the following data-driven estimate for the model error bound 
\begin{equation}
    \beta(\x, \a) := a \norm{\acc} + b \norm{\qdot(\x)} + c,
\end{equation}
which upper bounds roughly $95$ \% of the data. We illustrate the resulting conservative estimator in Figure \ref{fig:bounding_plane}.
\\\\
Having presented how we compute the model error constant, we next discuss how to select a $\matP$ and $\matK$ matrix. We approach this by first solving 50 instances of the optimization problem in Appendix \ref{sec:cvx_control_syn} with $\rho$ equally spaced between $0.5$ and $1.0$. Each solved instance results in matrices $\matP$ and $\matK$, from which we estimate the model error constants according to the previous discussion. Then, we pick an auxiliary controller by the selection rule specified in Appendix~\ref{sec:opt_candidate_selection}.
\subsubsection{Simulations}
\label{sec:dd:sim}
We use the same methods as specified in Section \ref{sec:benchmarks}. We apply a constant control input every $\timeStep=10$ [ms]. The robust MPC methods are optimized every $\NrAuxSteps=4$ steps. We evaluate the methods on 10 randomly generated corridors, where we sample 3 configurations uniformly within the configuration space and set the radius of the corridor to $0.2$ [rad].
\\\\
For the simulations, we randomly sample disturbances from the following multivariate Gaussian distribution
\def\matW{\textbf{W}}
\def\vecb{\textbf{b}}
\begin{equation}
    |\Delta(\x, \acc)| \sim \mathcal{N}(\matW_v |\qdot(\x)| + \matW_a |\acc| + \vecb, \: \bm{\Sigma}).
\end{equation}
In the above, $|\cdot|$ is the element-wise absolute value function. Thus, the mean is represented by a linear regression model, with weight matrices, $\matW_v \in \setR^{\DimState\times\DimConf}$ and $\matW_a \in \setR^{\DimState\times\DimConf}$, and bias term, $\b \in \setR^{\DimState}$. The parameters and the covariance matrix, $\Sigma \in \setR^{\DimState \times \DimState}$, are computed from the data set. The signs of the elements in the disturbance is assigned randomly with equal probability.
\subsubsection{Results}
\label{sec:dd:res}
The results of the experiments are presented in Table \ref{tab:dd_exp:results}.
\begin{table}[H]
    \centering
    \caption{Presents statistics from 10 randomly generated corridors.}
    \label{tab:dd_exp:results}
    \begin{tabular}{l|l l}
            & \multicolumn{2}{c}{Time to goal [s]} \\
    Method  & Mean  &  Std \\
    \hline
Flexible (ours)      &      9.04 &       1.35    \\
Rigid                &       - &       -    \\
Nominal              &       - &       -    \\
Oracle               &       2.17 &       0.39    \\
    \end{tabular}
\end{table}
The outcome mirrors the pattern observed in Section \ref{sec:results}. The oracle MPC, i.e., the nominal MPC without disturbances, is the fastest. However, it is an unrealistic benchmark, only used to form a lower bound on the time duration. If the same planner is used with disturbances, i.e., the nominal MPC, it ends up in an infeasible state, not being able to operate. The rigid tube is too conservative and cannot be executed due to that the constraint tightening is too large. Our method is able to operate in the presence of the disturbances, executing safe motion for all problem instances. }
\fi
\end{document}